\documentclass{article} 

\usepackage{iclr2027_conference,times}

\usepackage{amsmath,amsfonts,bm}

\def\eqref#1{equation~\ref{#1}}

\def\1{\bm{1}}

\DeclareMathAlphabet{\mathsfit}{\encodingdefault}{\sfdefault}{m}{sl}
\SetMathAlphabet{\mathsfit}{bold}{\encodingdefault}{\sfdefault}{bx}{n}

\usepackage{hyperref}
\usepackage{url}

\usepackage{algorithm}
\usepackage{algorithmic}
\usepackage{amsmath}
\usepackage{amssymb}
\usepackage{comment}
\usepackage{booktabs}
\usepackage{graphicx}
\usepackage{hyperref}
\usepackage{xcolor}
\usepackage{dcolumn}
\usepackage{tikz}
\usetikzlibrary{positioning,arrows.meta,calc}
\title{\ourlong}
\usepackage{cleveref}
\usepackage{wrapfig}
\usepackage{amsthm}
\theoremstyle{definition}
\newtheorem{definition}{Definition}
\newtheorem{proposition}{Proposition}
\usepackage{longtable}
\usepackage{soul}
\usepackage{url}
\usepackage[utf8]{inputenc}
\usepackage{xspace}
\usepackage{graphicx}
\usepackage{amsmath}
\usepackage{amssymb}
\usepackage{booktabs}
\usepackage[switch]{lineno}
\usepackage{xcolor}
\usepackage{listings}
\usepackage{float}
\usepackage{fontawesome5}
\usepackage{colortbl}
\usepackage{pgffor}
\usepackage{enumitem}
\usepackage{multirow}
\usepackage{arydshln}
\usepackage{mdframed}

\usepackage{pifont}

\usepackage{tikz}

\usepackage{circledsteps}

\newcommand{\Dset}{\mathcal{D}}                  
\newcommand{\Df}{\mathcal{D}_f}                  
\newcommand{\Dr}{\mathcal{D}_r}                  
\newcommand{\Alg}{\mathcal{A}}                   
\newcommand{\Unl}{\mathcal{U}}                   
\newcommand{\pth}{\pi_{\theta}}                  
\newcommand{\porig}{\pi_{\mathrm{o}}}            
\newcommand{\punl}{\pi_{\mathrm{u}}}             
\newcommand{\pretr}{\pi_{\mathrm{retr}}}         
\newcommand{\lossf}{\ell_{f}}                    
\newcommand{\lossr}{\ell_{r}}                    

\definecolor{aliceblue}{rgb}{0.94, 0.97, 1.0} 
\definecolor{azure(colorwheel)}{rgb}{0.0, 0.5, 1.0} 
\definecolor{aureolin}{rgb}{0.99, 0.93, 0.0} 

\lstdefinelanguage{json}{
    basicstyle=\ttfamily\scriptsize, 
    numbers=left, 
    numberstyle=\tiny\color{gray}, 
    numbersep=-8pt, 
    xleftmargin=0em, 
    columns=flexible, 
    breaklines=true,
    frame=single,
    backgroundcolor=\color{aureolin!4},
    literate=
      *{"class"}{{\textcolor{blue}{"class"}}}{1}
      {"parameters"}{{\textcolor{teal}{"parameters"}}}{1}
      {"partitions"}{{\textcolor{violet}{"partitions"}}}{1}
      {"data"}{{\textcolor{red}{"data"}}}{1}
      {"unlearners"}{{\textcolor{red}{"unlearners"}}}{1}
      {"predictor"}{{\textcolor{red}{"predictor"}}}{1}
      {"evaluator"}{{\textcolor{red}{"evaluator"}}}{1}
}

\lstdefinelanguage{myPython}[]{Python}{
    basicstyle=\ttfamily\scriptsize, 
    numbers=left, 
    numberstyle=\tiny\color{gray}, 
    xleftmargin=0em,
    breaklines=true,
    frame=single,
    backgroundcolor=\color{aureolin!4},
    stringstyle=\color{purple},
    commentstyle=\color{teal}
}

\mdfdefinestyle{stebox1}{%
    linecolor=black!70,
    linewidth=1.2pt,
    leftmargin=0cm,
    rightmargin=0cm,
    roundcorner=2pt,
    innerleftmargin=10pt,
    innerrightmargin=10pt,
    innertopmargin=6pt,
    innerbottommargin=6pt,
    topline=false,
    bottomline=false,
    rightline=true,
    backgroundcolor=gray!8
}

\usepackage{subcaption}

\usepackage{adjustbox}  
\usepackage[most]{tcolorbox}
\usepackage{xspace}
\newcounter{takeaway}

\newlength{\rqlabelwidth}
\newtcolorbox{rqinner}[1][]{%
  colback=black!5,          
  colframe=black!5,        
  boxrule=1pt,
  arc=0pt,
  left=6pt, right=6pt, top=5pt, bottom=5pt,
  boxsep=0pt,
  #1
}

\newcommand{\rqbox}[1]{%
  \par\medskip\noindent
  \refstepcounter{takeaway}%
  \adjustbox{valign=c}{%
    \makebox[\rqlabelwidth][c]{%
      \rotatebox{90}{\footnotesize\bfseries \#\thetakeaway}%
    }%
  }%
  \adjustbox{valign=c}{%
    \begin{minipage}{\dimexpr\linewidth-\rqlabelwidth\relax}
      \begin{rqinner}
        \itshape #1
      \end{rqinner}
    \end{minipage}%
  }%
  \par\medskip
}

\newcommand{\cta}[1]{%
  \par\medskip\noindent
  \refstepcounter{takeaway}%
  \adjustbox{valign=c}{%
    \makebox[\rqlabelwidth][c]{%
      \rotatebox{90}{\footnotesize\bfseries Call-to-Action}%
    }%
  }%
  \adjustbox{valign=c}{%
    \begin{minipage}{\dimexpr\linewidth-\rqlabelwidth\relax}
      \begin{rqinner}
        \itshape #1
      \end{rqinner}
    \end{minipage}%
  }%
  \par\medskip
}

\iclrfinalcopy

\author{Bardh Prenkaj\thanks{Equal contribution} \\
Sapienza University of Rome\\
\texttt{prenkaj@di.uniroma1.it}
\And
Andrea D'Angelo\footnotemark[1],\, Davide Mottin\\
Aarhus University\\
\texttt{\{andrea,davide\}@cs.au.dk}
\And
Federico Fontana, Davide Gabrielli, Paola Velardi\thanks{Paola Velardi is also affiliated with ISTC-CNR.},\, Stefano Faralli\\
Sapienza University of Rome\\
\texttt{\{fontana.f,gabrielli.d,velardi,faralli\}@di.uniroma1.it}
}

\newcommand\our[0]{CRU}
\newcommand\ourlong[0]{Causal Routing for Unlearning}

\begin{document}
 
\maketitle

\begin{abstract}
LLMs cannot forget the way we delete a file. Strangely, we are asked to remove something that was never put anywhere in particular. What the model took from a piece of text is now smeared across billions of weights. Existing methods rewrite all of them to change one thing, and none of them say which part produced that change. To address this, we introduce Causal Routing for Unlearning (CRU) by asking where the concept is expressed in the model and suppressing only that part. One untrained forward pass over the forget set ranks neurons by how their activations vary. Then, small routing modules on those neurons gate and suppress only the concepts that need to be forgotten. In CRU, the base model is frozen, and any change in behavior is caused only by the gated neurons; hence, why the routing is causal. Due to our parameter efficiency (only $\sim$$0.01\%$ as many parameters as the base model), unlearning a concept costs $14$ GiB, whereas the baselines require $71$ GiB. On TOFU, CRU is indistinguishable from the retained model ($p > 0.05$, KS test) and is never Pareto-dominated, whereas every compared baseline matches its forgetting on the larger-forget batches only by collapsing utility. On RWKU, it achieves an adversarial-probe recall of $0.052$, compared to $0.250$ for the strongest baseline, meaning the knowledge is gone, not merely harder to reach. Thus, deciding on the intervention at query time, rather than fixing it beforehand, is the axis along which we argue that unlearning should proceed.
\end{abstract}
 
\section{Introduction}
\label{sec:intro}

A model cannot forget (unlearn) the way we delete a file. If we delete a record from a database, we know exactly where it was: we delete the row, and the data disappears. A language model stores nothing at a given address. What it learned from a sentence is distributed across billions of weights, tangled with everything else it knows, and there is no row to strike. This is why the right to be forgotten \cite{gdpr_article17,ccpa_california_ag}, which is trivial to honor in a traditional database, becomes, inside a trained model, a genuinely strange request: \textit{we are asked to remove something that was never put anywhere in particular.}

The honest solution is to retrain the model on everything except the data we want to forget. Yet training an LLM end-to-end is infeasible for any small or medium-sized research lab or company. The only practical alternative is to navigate the model's weight and concept space and update it so that the forget set can no longer be retrieved.
 

Existing methods tackle this by pushing probability mass off the forget set: by maximizing loss on it \citep{yao2024large}, saturating that objective so the model does not diverge \citep{npo,simnpo}, redirecting the query to a refusal \citep{rafailov2023direct}, acting on the context alone \citep{pawelczyk2023context}, or rewarding the absence of the concept \citep{zaradoukas2026reinforcement}. The methods share a common trait: \textit{they rewrite the whole model to change one thing}. Even when the change is the intended one, nothing in the procedure tells us which part of the network produced it, so the result is expensive to obtain, hard to reconstruct, and hard to reuse for a second concept.

A second line of work leaves the weights untouched. NeuMuter \citep{hou2025decoupling} trains a mask over a small set of neurons with the model frozen, but the mask is fixed once trained, so it applies to every input alike. DSG \citep{muhamed2025saes} and GUARD-IT \citep{turani2026inference} do respond to the input, but only through a predefined threshold and an intervention magnitude. This second family shares the same shortcoming as the first: \textit{the intervention is decided before the model is ever queried}. In other words, there is no mechanism distinguishing a prompt that reaches for the forgotten concept from one that merely passes nearby, so \textit{suppression is spent uniformly whether or not it is needed} (Figure~\ref{fig:positioning}). Consequently, model utility drops lower than it needs to.
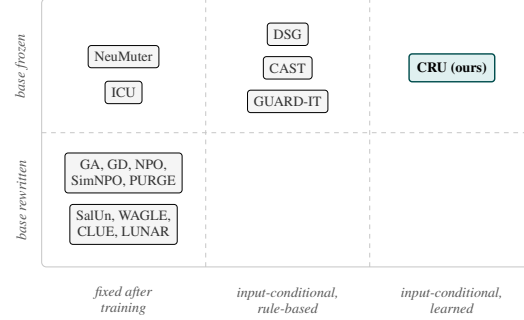
\begin{wrapfigure}{r}{0.5\textwidth}
\centering
\resizebox{\linewidth}{!}{%
\begin{tikzpicture}[
  font=\small,
  meth/.style={draw, rounded corners=2pt, line width=0.4pt, inner sep=4pt,
               align=center, fill=gray!8, text=black!80, font=\footnotesize},
  ours/.style={draw, rounded corners=2pt, line width=0.9pt, inner sep=5pt,
               align=center, fill=teal!12, draw=teal!60!black, text=black,
               font=\footnotesize\bfseries},
  axlab/.style={font=\footnotesize\itshape, text=black!65, align=center}
]
\draw[gray!45, line width=0.4pt, rounded corners=3pt] (0,0) rectangle (11,6);
\draw[gray!45, line width=0.4pt, dashed] (3.667,0) -- (3.667,6);
\draw[gray!45, line width=0.4pt, dashed] (7.333,0) -- (7.333,6);
\draw[gray!45, line width=0.4pt, dashed] (0,3) -- (11,3);

\node[axlab, rotate=90] at (-0.45,4.5) {base frozen};
\node[axlab, rotate=90] at (-0.45,1.5) {base rewritten};
\node[axlab] at (1.83,-0.75) {fixed after\\training};
\node[axlab] at (5.50,-0.75) {input-conditional,\\rule-based};
\node[axlab] at (9.17,-0.75) {input-conditional,\\learned};

\node[meth] at (1.83,4.70) {NeuMuter};
\node[meth] at (1.83,3.90) {ICU};

\node[meth] at (5.50,5.2) {DSG};
\node[meth] at (5.50,4.45) {CAST};
\node[meth] at (5.50,3.70) {GUARD-IT};

\node[ours] at (9.17,4.45) {CRU (ours)};

\node[meth] at (1.83,2.10) {GA, GD, NPO,\\SimNPO, PURGE};
\node[meth] at (1.83,0.95) {SalUn, WAGLE,\\CLUE, LUNAR};
\end{tikzpicture}
}
\caption{Unlearning methods positioned by whether the base model is modified or not; coarsely inspired from \cite{yoon2026position}.}
\label{fig:positioning}
\end{wrapfigure}

To enforce the distinction between concepts involved in the forget set and the rest, our method, \ourlong\ (\our), splits unlearning into \emph{localization}, i.e., identifying the neurons that encode forget-set concepts, and \emph{gating} of their activations. 

\noindent\textbf{Localization.}
\our\ collects hidden states from the last few decoder blocks over a sample of the forget set.  We note the tension with weight-editing work that locates factual associations in middle layers \citep{meng2022locating}: those methods intervene on where a fact is \emph{retrieved}, whereas \our\ intervenes on where it is \emph{expressed}. Hence, \our\ pools each example over its non-padding positions and ranks neurons by how much that pooled activation varies from one example to the next; neurons above a percentile threshold become the candidate set for that block. This step determines \emph{where} \our\ may intervene, restricting it to a small fraction of neurons in a few blocks, and entails only one forward pass and no training. It does not by itself make the intervention specific to the forget set; that is enforced by the gating objective (\Cref{eq:gating}), which penalizes closing gates on retain inputs and thus determines \emph{when} each selected neuron is suppressed.

\noindent\textbf{Gating.} At each of those blocks, \our\ attaches a small MLP ($\sim$0.01\% of total parameters of the base model) that reads the block's output hidden state and emits, for every selected neuron at every token position, a value in $[0,1]$ that multiplies it. Training pushes those values toward $0$ for forget-set inputs and toward $1$ for the remaining inputs; only the MLPs receive gradients, and the objective is defined on the gate values themselves rather than on the model's output distribution. Because the base model remains entirely frozen, there is no parameter drift, and all behavioral changes are attributable to the gated neurons.

\section{Preliminaries}\label{sec:preliminaries}

\textbf{Models and data.}
We consider an autoregressive language model $\pth(y \mid x)$ with parameters $\theta \in \Theta$, mapping prompt $x$ to response distribution $y$ via negative log-likelihood loss:
\begin{equation}
    \ell(y \mid x; \theta) \;=\; -\log \pth(y \mid x) \;=\; -\sum\nolimits_{t=1}^{|y|} \log \pth\!\left(y_t \mid x, y_{<t}\right).
\end{equation}
A learning algorithm $\Alg$ maps a dataset to parameters. For training set $\Dset$, the reference (original) model parameters before unlearning are $\theta_{\mathrm{o}} = \Alg(\Dset)$, yielding distribution $\porig$.
 
\textbf{Machine unlearning.}
Let $\Df \subseteq \Dset$ be the \emph{forget set}, the data whose influence is to be removed, and $\Dr = \Dset \setminus \Df$ the \emph{retain set}. Machine unlearning asks for a model that behaves as though $\Df$ had never been seen~\citep{cao2015towards}. The construction that achieves this by definition is retraining.
 
\begin{definition}[Exact unlearning]\label{def:retrain}
Given learning algorithm $\Alg$ and retain set $\Dr$, the \emph{retrained model} $\pretr$ has parameters $\theta_{\mathrm{retr}} = \Alg(\Dr)$.
\end{definition}%
$\pretr$ serves as the \emph{golden reference point} by excluding $\Df$. Since retraining $\Alg(\Dr)$ is computationally intractable for LLMs, unlearning aims to approximate $\pretr$ directly \citep{bourtoule2021machine}.

\begin{definition}[Approximate unlearning]\label{def:unlearning}
Given $\porig$, $\Df$, and $\Dr$, an unlearning algorithm $\Unl$ outputs an \emph{unlearned model} $\punl = \Unl(\porig, \Df, \Dr) \approx \pretr$ without retraining.
\end{definition}

While strict definitions demand $(\epsilon, \delta)$-indistinguishability between $\punl$ and $\pretr$ \citep{guo2020certified, sekhari2021remember}, such certificates are impractical for non-convex LLMs trained via stochastic optimization \citep{thudi2022necessity}. Furthermore, real-world deployments lack full access to $\Dset$, $\Dr$, and $\pretr$. Evaluation, therefore, relies on proxy metrics measuring \emph{forget quality} on $\Df$ and \emph{utility retention} on $\Dr$ and general benchmarks \citep{maini2024tofu,shi2025muse}.

\textbf{What $\Unl$ is allowed to do.}
Existing unlearning methods differ along two independent axes: the \emph{objective} optimized and the \emph{scope} of the intervention. 

\emph{Objective.} Unlearning typically optimizes a regularized loss balancing a forget term on $\Df$ and forget loss $\ell_f$ against a utility retention term on $\Dr$ and loss $\ell$ \citep{liu2025rethinking}:
\begin{equation}\label{eq:unlearning-objective}
    \min_{\theta \in \Theta} \;\;
    \underbrace{\mathbb{E}_{(x,y) \sim \Df}\!\left[\lossf(y \mid x; \theta)\right]}_{\text{forget}}
    \;+\; \lambda \,
    \underbrace{\mathbb{E}_{(x,y) \sim \Dr}\!\left[\ell(y \mid x; \theta)\right]}_{\text{retain}},
\end{equation}
where $\lambda > 0$ controls the trade-off, and state-of-the-art methods vary choices of $\lossf$ and $\lambda$ (\S\ref{app:sota}).

\emph{Scope.} While \Cref{eq:unlearning-objective} permits updates across all of $\Theta$, any parameter-space unlearning operator can be generalized as
\begin{equation}\label{eq:scope}
    \theta_{\mathrm{u}} \;=\; \theta_{\mathrm{o}} + \mathbf{w} \odot \Delta,
    \qquad \mathbf{w} \in [0,1]^{d},
\end{equation}
where $\Delta$ is the objective update, $d = \dim \Theta$, $\odot$ denotes elementwise multiplication, and $\mathbf{w}$ selects which coordinates receive updates.

\section{Related Work}
\label{sec:related}

\subsection{Base Model Update}

\textbf{Global updates.} Gradient Ascent (GA) \citep{yao2024large} performs gradient ascent on the cross-entropy loss of the forget set, lowering its likelihood; the objective is unbounded, so parameters diverge and coherence collapses. Gradient Difference (GD) \citep{liu2022continual} adds a retain accuracy penalty to balance this. 
DPO \citep{rafailov2023direct} trains the model to prefer a refusal (e.g., ``I don't know'') over the original answer to each forget question. NPO \citep{npo} pushes down the probability of the original answers, measured relative to the model before unlearning, thereby keeping the loss bounded. SimNPO \citep{simnpo} drops this comparison with the original model, which otherwise spends the same effort on every example regardless of how hard it is to forget. PURGE \citep{zaradoukas2026reinforcement} uses RL where the model is penalized whenever its output mentions the target concept, while a KL term keeps it close to the original one to preserve utility. Since the reward only checks what the model says, it remains unclear whether the knowledge is erased or merely hidden. All global methods fix $\mathbf{w} = \mathbf{1}$, thus disregarding the input. 

\textbf{Localized updates.} A second group edits only a subset of the weights (i.e., $\mathbf{w} \neq \mathbf{1}$). SalUn \citep{fan2024salun} updates only the weights whose gradients on the forget set are largest. WAGLE \citep{jia2024wagle} estimates how much each weight contributes to unlearning and applies standard objectives (e.g., GradDiff, NPO) only to the most influential ones. CLUE \citep{chen2025clue} traces the circuits used by the forget set, separates neurons specific to it from those shared with other knowledge, and fine-tunes only the former. LUNAR \citep{lunar} edits the MLP output projections of a few layers so that forget inputs produce activations resembling those of a refusal. Crucially, while these methods set $\mathbf{w} \neq \mathbf{1}$, their modified weights apply indiscriminately to every input.

\subsection{Frozen Base Model}
\label{sec:frozen}
\noindent\textbf{Input-independent interventions.} NeuMuter \citep{hou2025decoupling} trains a fixed mask over $\sim$$1\%$ of feed-forward neurons on frozen weights, applying it uniformly to all inputs. Parameter-free ICU \citep{pawelczyk2023context} prepends label-flipped examples, but context-bound interventions can be bypassed when reasoning traces reconstruct forgotten concepts \citep{zaradoukas2026reinforcement}.

\noindent\textbf{Rule-based conditional interventions.} Rule-based methods steer activations only when heuristic triggers fire. CAST \citep{lee2025programming} compares hidden states to a condition vector, applying a fixed steering vector to all generated tokens if a grid-searched threshold is met on the prompt. DSG \citep{muhamed2025saes} and GUARD-IT \citep{turani2026inference} follow the same paradigm, clamping sparse autoencoder features or prototype directions to fixed magnitudes upon meeting tuned thresholds. In all cases, trigger thresholds and intervention magnitudes are user-defined hyperparameters.

We refer the reader to \S\ref{app:objective} for details on how the SoTA operates under \Cref{eq:unlearning-objective}, and how \our\ differs from it according to \Cref{eq:scope} (see \S\ref{app:scope}).

\section{\ourlong}
\label{sec:cru}
\begin{wrapfigure}{r}{0.5\linewidth}
\centering
\begin{tikzpicture}[
  font=\scriptsize,
  blk/.style={draw, rounded corners=1pt, minimum width=0.24cm,
              minimum height=0.44cm, line width=0.3pt, fill=gray!12},
  gblk/.style={blk, minimum width=0.30cm, fill=teal!20, draw=teal!60!black},
  rsq/.style={draw, rounded corners=1pt, minimum width=0.26cm,
              minimum height=0.24cm, line width=0.35pt,
              fill=violet!18, draw=violet!60!black},
  box/.style={draw, rounded corners=2pt, line width=0.4pt, align=center,
              inner sep=2.5pt, fill=gray!6},
  rl/.style={draw, rounded corners=2pt, line width=0.45pt, align=center,
             inner sep=1.2pt, minimum width=1.45cm, minimum height=0.28cm,
             fill=violet!12, draw=violet!60!black},
  cl/.style={draw, rounded corners=1pt, minimum width=0.30cm,
             minimum height=0.30cm, line width=0.3pt, fill=gray!10},
  ccl/.style={cl, fill=teal!18, draw=teal!60!black},
  scl/.style={cl, fill=orange!12, draw=orange!70!black},
  gcl/.style={draw, rounded corners=1pt, minimum width=0.46cm,
              minimum height=0.34cm, line width=0.3pt},
  ar/.style={-latex, line width=0.35pt, gray!75},
  lb/.style={text=black!62}
]

\foreach \x in {0.35,0.62,0.89} \node[blk] at (\x,10.05) {};
\node[lb] at (1.10,10.05) {$\cdots$};
\foreach \x in {1.34,1.66,1.98,2.30} \node[gblk] at (\x,10.05) {};
\foreach \x in {1.34,1.66,1.98,2.30} {
  \node[rsq] at (\x,9.48) {};
  \draw[ar] (\x,9.83) -- (\x,9.60);
}
\draw[gray!45, line width=0.3pt, dashed, rounded corners=2pt]
  (0.20,9.79) rectangle (2.52,10.31);
\node[lb] at (1.36,10.56) {frozen blocks $1\dots L$};
\node[lb, anchor=west, text=violet!45!black] at (2.70,9.48) {routers (trained)};
\draw[ar] (2.52,10.05) -- (2.86,10.05);
\node[box] (hd) at (3.55,10.05) {LM head};
\draw[ar] (4.24,10.05) -- (4.58,10.05);
\node[box] (lg) at (5.20,10.05) {logits};

\draw[gray!50, line width=0.3pt, dashed] (1.30,9.34) -- (0.15,8.96);
\draw[gray!50, line width=0.3pt, dashed] (2.34,9.34) -- (6.40,8.96);

\draw[gray!40, line width=0.4pt, rounded corners=3pt] (0.15,4.70) rectangle (6.40,8.96);
\node[lb, anchor=west] at (0.30,8.75) {inside a selected block};

\node[box, minimum width=2.2cm] (bl) at (1.45,8.15) {self-attention\\MLP};
\node[lb, anchor=west] at (2.70,8.15) {frozen};
\draw[ar] (1.45,7.75) -- (1.45,7.45);

\foreach \x/\s in {0.50/cl, 0.88/ccl, 1.26/cl, 1.64/ccl, 2.02/cl, 2.40/ccl}
  \node[\s] at (\x,7.25) {};
\node[lb, anchor=west] at (0.35,7.62) {$\mathbf{h}_l$};
\node[lb, anchor=west] at (2.72,7.02) {\textcolor{teal}{$\mathcal{C}_l$}};

\node[lb, text=violet!45!black] at (4.95,7.52) {$R_l$ (trained)};
\draw[ar] (2.58,7.25) -- (3.60,7.25) -- (3.60,7.10) -- (4.22,7.10);
\node[rl] (l1) at (4.95,7.10) {\texttt{Linear}};
\node[rl] (ac) at (4.95,6.76) {\texttt{ReLU}};
\node[rl] (l2) at (4.95,6.42) {\texttt{Linear}};
\node[rl] (sg) at (4.95,6.08) {\texttt{Sigmoid}};
\draw[gray!45, line width=0.3pt, dashed, rounded corners=2pt]
  (4.12,5.90) rectangle (5.78,7.28);

\draw[ar] (1.45,7.10) -- (1.45,5.95);
\draw[ar] (4.95,5.90) -- (4.95,5.78) -- (1.67,5.78);
\node[lb] at (3.10,5.98) {$\mathbf{w}_l \in [0,1]^|$\textcolor{teal}{$^{\mathcal{C}_l}$}$^|$};
\node[draw, circle, inner sep=1pt, line width=0.4pt] at (1.45,5.78) {$\odot$};

\draw[ar] (1.45,5.61) -- (1.45,5.41);
\foreach \x/\s in {0.50/cl, 0.88/scl, 1.26/cl, 1.64/scl, 2.02/cl, 2.40/scl}
  \node[\s] at (\x,5.24) {};
\node[lb, anchor=west] at (2.62,5.24) {$\mathbf{h}_l'$};
\draw[ar] (1.45,5.07) -- (1.45,4.85);
\node[lb, anchor=west] at (1.62,4.88) {to block $l{+}1$};

\draw[gray!40, line width=0.4pt, rounded corners=3pt] (0.15,2.20) rectangle (6.40,4.45);
\node[lb, anchor=west] at (0.30,4.22) {two example inputs};

\node[lb, anchor=east] at (0.95,3.68) {forget};
\foreach \x/\v in {1.25/0.02, 2.10/0.11, 2.95/0.00}
  \node[gcl, fill=red!18, draw=red!65!black] at (\x,3.68) {\v};
\node[lb, anchor=east] at (0.95,2.96) {retain};
\foreach \x/\v in {1.25/0.97, 2.10/0.94, 2.95/0.99}
  \node[gcl, fill=green!12, draw=green!55!black] at (\x,2.96) {\v};
\node[lb, anchor=west] at (3.25,3.32) {$\mathbf{w}_l$};

\node[box, align=center] at (5.05,3.32)
  {$\mathcal{L}_{\text{total}}=$\\
   $\mathcal{L}_{\mathrm{suppress}}+\lambda\mathcal{L}_{\mathrm{pass}}$\\[1pt]
   on gate values,\\not on logits};
\end{tikzpicture}
\caption{\textbf{\our\ applied to a frozen decoder.} Each of the last $k$ blocks pairs with a routing MLP, which reads the output hidden state $\mathbf{h}_l$ and emits a gate in $[0,1]$ for every selected neuron at every token position. Gates multiply the $\mathcal{C}_l$ coordinates of $\mathbf{h}_l$; all other coordinates pass through at $1$. Training minimizes the mean gate value on forget batches and drives it to one on retain batches: the objective is defined on the gate outputs, so no gradient passes through the base model.}
\label{fig:cru}
\end{wrapfigure}
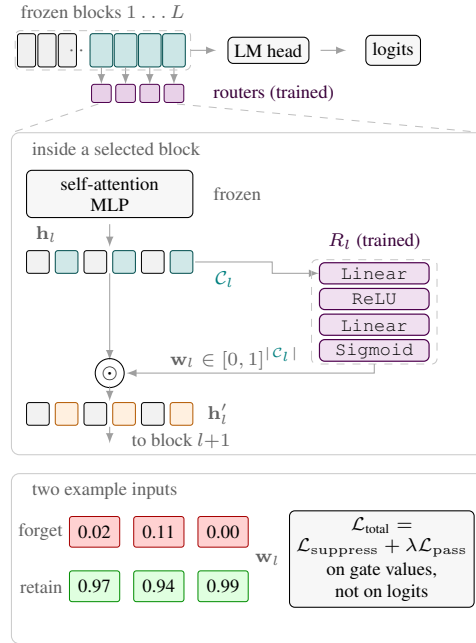
\our\footnote{\url{https://anonymous.4open.science/r/cru-iclr-4E56/README.md}} is very simple (see \Cref{fig:cru}). It operates in two phases. The first selects, for each block in a small set, the neurons whose activations vary most over the forget set. The second attaches a small routing network to each of those blocks and trains it to attenuate those neurons on forget-set inputs while leaving them intact otherwise. The base model is frozen in both phases.

\subsection{Phase 1: Localizing Concept Neurons}

\textbf{Which blocks.} Earlier layers in transformers carry surface and syntactic features, while later layers carry semantic ones \citep{tenney2019bert}; feed-forward layers act as key-value memories whose outputs refine the distribution over the vocabulary \citep{geva2021transformer}; and neurons expressing a given fact are concentrated in the last layers \citep{dai2022knowledge}. Decoding methods exploit the same asymmetry, contrasting late against early layers to amplify factual content \citep{chuang2024dola}. Late blocks are, therefore, where a concept is most nearly expressed in the form the output head consumes, which is what a multiplicative gate can act on. 
We, thus, collect activations from the last $k$ decoder blocks.

\noindent\textbf{Pooling.} For a forget-set example $i$ with token mask ${m}^{(i)}_t \in \{0,1\}$, we take the mean hidden state of block $l$ over non-padding positions,
\begin{equation}\label{eq:pooling}
\mathbf{\bar{a}}^{(i)}_l \;=\; \frac{\sum_t {m}^{(i)}_t \, \mathbf{h}^{(i)}_{l,t}}{\sum_t {m}^{(i)}_t}, \qquad \mathbf{h}^{(i)}_{l,t} \in \mathbb{R}^{d},
\end{equation}%
where $\mathbf{h}_{l,t}$ is the output hidden state of block $l$ at position $t$ and $d$ is the model width. Padding is excluded throughout: sequence length is an artifact of batching, and averaging over it would let downstream statistics key on how much a corpus was padded rather than on its content.

\noindent\textbf{Variance as a relevance score.} For each block, we score neuron $j$ by its variance across forget-set examples,
\begin{equation}
\sigma^2_l(j) \;=\; \operatorname{Var}_i\!\left(\mathbf{\bar{a}}^{(i)}_l[j]\right),
\end{equation}%
estimated over a fixed number of forget batches. A neuron whose pooled activation barely moves across examples of the concept is unlikely to be carrying it. This is a cheap surrogate for the attribution methods used elsewhere in the localization literature -- integrated gradients \citep{dai2022knowledge}, task predictivity \citep{wang2022finding}, activation probability \citep{tang2024language}, or neuron-level attribution \citep{yu2024neuron}. It requires a single pass and no backward computation.

\noindent\textbf{Percentile thresholding.} We keep the neurons above the $p$-th percentile of the per-block variance distribution,
\begin{equation}
\mathcal{C}_l \;=\; \bigl\{\, j : \sigma^2_l(j) > q_p\!\left(\sigma^2_l\right) \bigr\}.
\end{equation}%
A percentile is robust to the shape of the variance distribution, which matters here. In particular, activations in large transformers are dominated by a small number of high-magnitude outlier dimensions \citep{timkey2021rogue,dettmers2022int8}, so a threshold defined by mean and standard deviation would be set almost entirely by those coordinates. 

\subsection{Phase 2: Learned Routing}\label{sec:phase2}
Phase~1 asks which neurons carry the concept, and answers it with a statistic pooled over each example. Phase~2 asks a narrower question at every token: \textit{given what is passing through this position now, should those neurons be closed?} Selection is therefore made once per concept, while the intervention is decided per position.
 
\noindent\textbf{Where the gate acts.} Each selected block receives its own routing module, applied to that block's output hidden state -- that is, to the residual stream, the channel through which blocks communicate \citep{elhage2021mathematical} and the site at which activation-space interventions are conventionally applied \citep{turner2023steering}. Acting there rather than inside the block keeps the intervention independent of the block's internal parameterization, so the same construction applies unchanged across the model families we test, even though their feed-forward blocks differ.

\noindent\textbf{Router architecture.} Let $\mathcal{C}_l = \{j_1 < j_2 < \dots < j_{|\mathcal{C}_l|}\}$ for the selected neurons of block $l$, indexed in increasing order. For that block, we define $R_l : \mathbb{R}^{d} \to [0,1]^{|\mathcal{C}_l|}$,
\begin{equation}\label{eq:router}
    \mathbf{w}_{l,t} \;=\; R_l\!\left(\mathbf{h}_{l,t}\right) \;=\; \sigma\!\left(\mathbf{W}_2 \,\mathrm{ReLU}\!\left(\mathbf{W}_1 \mathbf{h}_{l,t} + \mathbf{b}_1\right) + \mathbf{b}_2\right),
\end{equation}
with $\mathbf{W}_1 \in \mathbb{R}^{d_r \times d}$,
$\mathbf{b}_1 \in \mathbb{R}^{d_r}$, $\mathbf{W}_2 \in \mathbb{R}^{|\mathcal{C}_l| \times d_r}$, $\mathbf{b}_2 \in \mathbb{R}^{|\mathcal{C}_l|}$, bottleneck width $d_r << d$. The same parameters are applied independently at every position $t$, so the module costs $|\mathcal{C}_l|$ outputs per token rather than $d$. The entry $\mathbf{w}_{l,t}[m]$ is the gate on neuron $j_m$; this correspondence between output row and neuron is a stipulation of the construction, a point we return when discussing the objective.

The gate is applied multiplicatively and only on coordinates corresponding to the neurons selected in Phase 1, yielding $\mathbf{h}'_{l,t} \in \mathbb{R}^{d}$ with
\begin{equation}\label{eq:gating}
\mathbf{h}'_{l,t}[j] \;=\;
\begin{cases}
\mathbf{w}_{l,t}[m] \cdot \mathbf{h}_{l,t}[j], & j = j_m \in \mathcal{C}_l,\\[2pt]
\mathbf{h}_{l,t}[j], & j \notin \mathcal{C}_l .
\end{cases}
\end{equation}%
\begin{wrapfigure}{r}{0.5\textwidth}
    \centering
    \includegraphics[width=\linewidth]{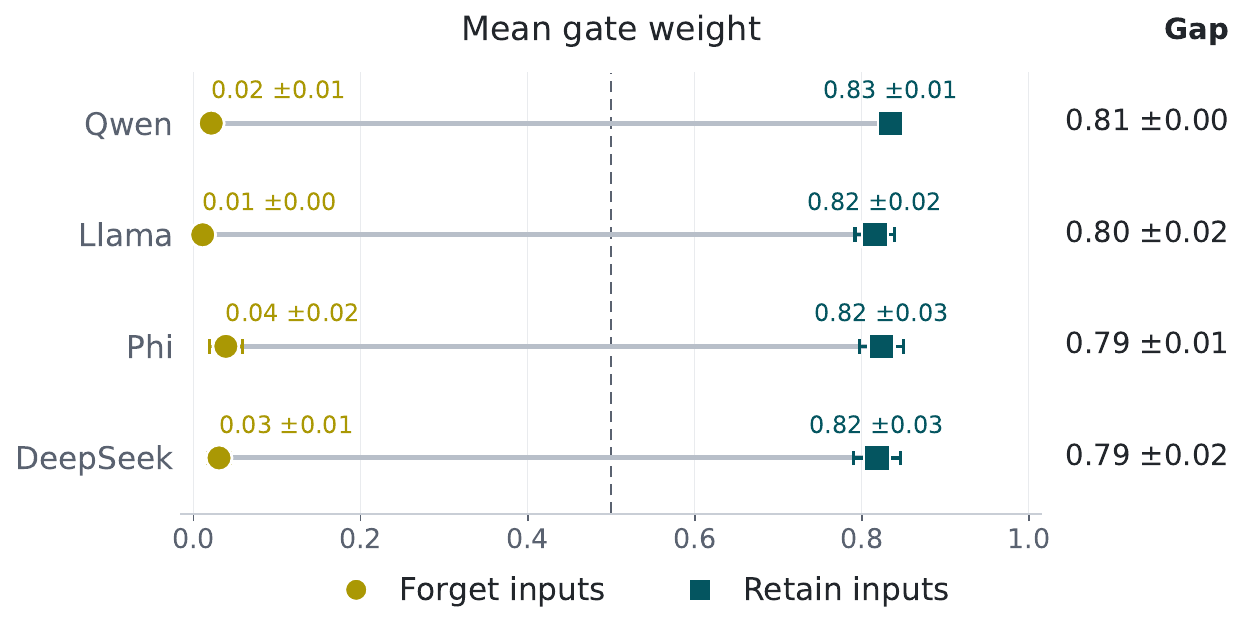}
    \caption{Mean routing gate on the forget and retain corpora, per model on \texttt{TOFU} \texttt{forget10}, averaged over 3 runs.}
    \label{fig:gate_gap}
\end{wrapfigure}
Three properties follow. The gate is \underline{multiplicative}, which is the argument for activation scaling over additive steering \citep{stoehr2024activation}; sigmoid gating of this form is standard in transformer feed-forward design \citep{dauphin2017language,shazeer2020glu}. It is \underline{soft}: because each $\mathbf{w}_{l,t}[m] \in [0,1]$ is continuous, the intervention is differentiable and admits partial suppression, unlike a hard mask. And it is \underline{input-conditional} and evaluated per token, so the same neuron may be suppressed at one position and passed at another. Together with a frozen base and a small trained module, this places \our\ in the same family as adapter-style parameter-efficient methods \citep{houlsby2019parameter,hu2022lora}, and $R_l$ plays the role a router plays in conditional computation \citep{shazeer2017outrageously}.

\noindent\textbf{Objective.} Let $\bar{w}^{\,\mathrm{f}}_l$ and $\bar{w}^{\,\mathrm{r}}_l$ denote the mean gate value of block $l$, taken jointly over the selected neurons and the non-padding positions of a forget and a retain batch respectively; padded positions are excluded, since the router is never asked to act on them at inference. With $K$ the set of gated blocks, we minimize
\begin{equation}\label{eq:objective}
\mathcal{L}_{\text{total}} \;=\; \underbrace{\frac{1}{|K|}\sum_{l \in K}
\bar{w}^{\,\mathrm{f}}_l}_{\mathcal{L}_\text{suppress}}\;+\;\lambda \,
\underbrace{\frac{1}{|K|}\sum_{l \in K}\left(\bar{w}^{\,\mathrm{r}}_l -
1\right)^{2}}_{\mathcal{L}_\text{pass}}.
\end{equation}%
We motivate the asymmetry between the two terms in \S\ref{app:loss-shape}. Two things follow from \Cref{eq:objective}. First, the objective is defined in terms of the gate values themselves and never observes the model's output. Hence, no gradient passes through the language-model head, and the base weights are untouched by construction rather than by choice. Second, the objective depends on the router only through the two scalars $\bar{w}^{\,\mathrm{f}}_l$ and $\bar{w}^{\,\mathrm{r}}_l$: any reallocation of gate mass over the selected coordinates and token positions that leaves those means fixed leaves the loss unchanged. Permutations of the router's output rows are included. The correspondence between row $m$ and neuron $j_m$ is stipulated by \Cref{eq:gating}, but \Cref{eq:objective} gives it no gradient. Hence, two routers differing by such a permutation are equally optimal while suppressing different neurons.

\noindent\textbf{What the objective can and cannot do.} The two terms are in tension only when the router can tell the two corpora apart. If it cannot, they resolve to a single value.
\begin{proposition}
\label{prop:tied}
Suppose a gate is constrained to take the same value $w$ on forget and retain inputs. Then $\mathcal{L}(w) = w + \lambda (w-1)^2$ is minimized over $\mathbb{R}$ at $w^\star = 1 - \tfrac{1}{2\lambda}$. Consequently $w^\star \le 0$ whenever $\lambda \le \tfrac{1}{2}$, and $w^\star = \tfrac{1}{2}$ at $\lambda = 1$.
\end{proposition}

\begin{proof}
$\mathcal{L}'(w) = 1 + 2\lambda(w-1) = 0$ gives $w^\star = 1 - 1/(2\lambda)$, and $\mathcal{L}''(w) = 2\lambda > 0$
\end{proof}
\Cref{prop:tied} shows that a collapsed router produces equal gate means of $1 - 1/(2\lambda)$ ($0.5$ at default $\lambda = 1$). Our empirical results diverge sharply: on \texttt{TOFU} \texttt{forget10}, gate values average $0.01$--$0.06$ on forget inputs (effectively disabling selected neurons) versus $0.79$--$0.85$ on retain inputs (\Cref{fig:gate_gap}). Thus, \our\ dynamically reads the input rather than applying fixed attenuation, distinguishing it from frozen-base methods (\S\ref{sec:related}). The retained mean ($0.82$) reflects minor activation withholding on retain inputs; unlike weight editing, this cost is naturally bounded because underlying parameters remain untouched (\S\ref{sec:results}). Finally, \Cref{prop:tied} confirms that higher retain pressure ($\lambda$) shifts the baseline optimum toward $1$, sacrificing retain attenuation first.

\section{Experiments}\label{sec:experiments}

\noindent\textbf{Benchmarks and baselines.} We evaluate \our\ on \texttt{TOFU} \citep{maini2024tofu} (\texttt{forget01} and \texttt{forget10} splits) and \texttt{RWKU} \citep{cao2024rwku}; metrics are summarized in Table~\ref{tab:benchmark_metrics}. We compare against eight SoTA baselines: GA, GD, NPO \citep{npo}, SimNPO \citep{simnpo}, ICU \citep{pawelczyk2023context}, PURGE \citep{zaradoukas2026reinforcement} (on \texttt{RWKU}), and LUNAR \citep{lunar} (via its \textit{deviation score} on its custom \texttt{TOFU} split; \S\ref{apx:full_exp_settings}). Unlike prior work targeting single architectures, we test across four model families: DeepSeek 7B, Qwen2.5 7B, Llama 3.1 8B, and Phi 3 Mini (\S\ref{apx:full_exp_settings}). \textbf{In summary, we evaluate \our\ on 2 benchmarks (4 forget sets), 4 model families, and against 7 baselines.}

\noindent\textbf{Hyperparameters.} We use the default hyperparameters for the baselines on \texttt{TOFU}  and \texttt{RWKU}. For PURGE, we use its published model checkpoints. For \our, we use the recommended learning rate for \texttt{TOFU} and perform a small grid search for \texttt{RWKU} (\S\ref{apx:hyperparameters}). We set $k=4$ and $d_r = 32$ for TOFU, $d_r=256$ for RWKU.

\subsection{Results}\label{sec:results}

\rqbox{\our\ has the best Model Utility-Forgetting trade-off on \texttt{TOFU} splits (statistically significant and not Pareto-dominated). On \texttt{RWKU}, \our\ is the only one that removes the target knowledge against all attack levels.}

\textbf{CRU achieves the best utility-forgetting trade-off across all forget sizes on TOFU.} \Cref{fig:pareto_grid} plots Model Utility against Forgetting (Rouge-L), where top-right is the ideal performance. On \texttt{forget01}, \our\ outperforms all baselines on Llama, Phi, and Qwen, while matching top baselines (NPO, SimNPO) on DeepSeek.  On \texttt{forget10}, high Forgetting scores for GA, NPO, and SimNPO are artifacts caused by utility collapsing to $0.00$--$0.21$ (\Cref{fig:pareto_grid}). For example, on Llama, NPO/SimNPO reports $0.97$ forgetting score at $0.00$ utility, whereas \our\ achieves $0.53$ utility. Overall, \our\ maintains the best trade-off without destroying generation capabilities. This trade-off advantage is statistically significant: i.e., a Friedman test \citep{demsar} across all (model, split) blocks rejects equivalence ($p \in [2.2\mathrm{e}^{-4}, 3.5\mathrm{e}^{-3}]$), with \our\ holding the top mean rank ($1.20$--$1.80$) and beating every baseline in post-hoc Wilcoxon tests ($p \le 0.039$).  Crucially, \our\ is never Pareto-dominated. 

Additionally, we found out that \citet{lunar} use a bespoke forget set split on \texttt{TOFU}. For completeness, we compare \our\ to LUNAR according to their introduced deviation score (distance from ideal point; lower is better) on their custom split. \our\ achieves $7.1$ on Qwen (averaged over 3 runs) compared to LUNAR's $13.2$ and baselines' $50.4$--$92.9$ (\Cref{fig:lunar_comparison}).

\noindent\textbf{\our\ guarantees knowledge removal rather than hiding.}
While \texttt{TOFU} evaluates fictitious targets, \texttt{RWKU} tests real-world figures entangled with general knowledge. \Cref{fig:rwku-main} compares \our\ against baselines across 20 targets on Phi-3-mini-4k-instruct. Unlike other methods that merely suppress surface forms, \our\ is the only approach that truly removes target knowledge across all three probing levels: fill-in-the-blank (Level 1), QA (Level 2), and adversarial prompts (Level 3). Averaged over all levels, \our\ reduces forget ROUGE-L recall from $0.532$ (original model) to $0.048$, outperforming SoTA and resisting adversarial extraction. Alas, this comes with a modest drop in utility on MMLU and a slightly higher vulnerability to membership inference attacks, especially in generation (\S\ref{apx:full_results}).

\begin{figure}[t]
    \centering
    \includegraphics[width=\linewidth]{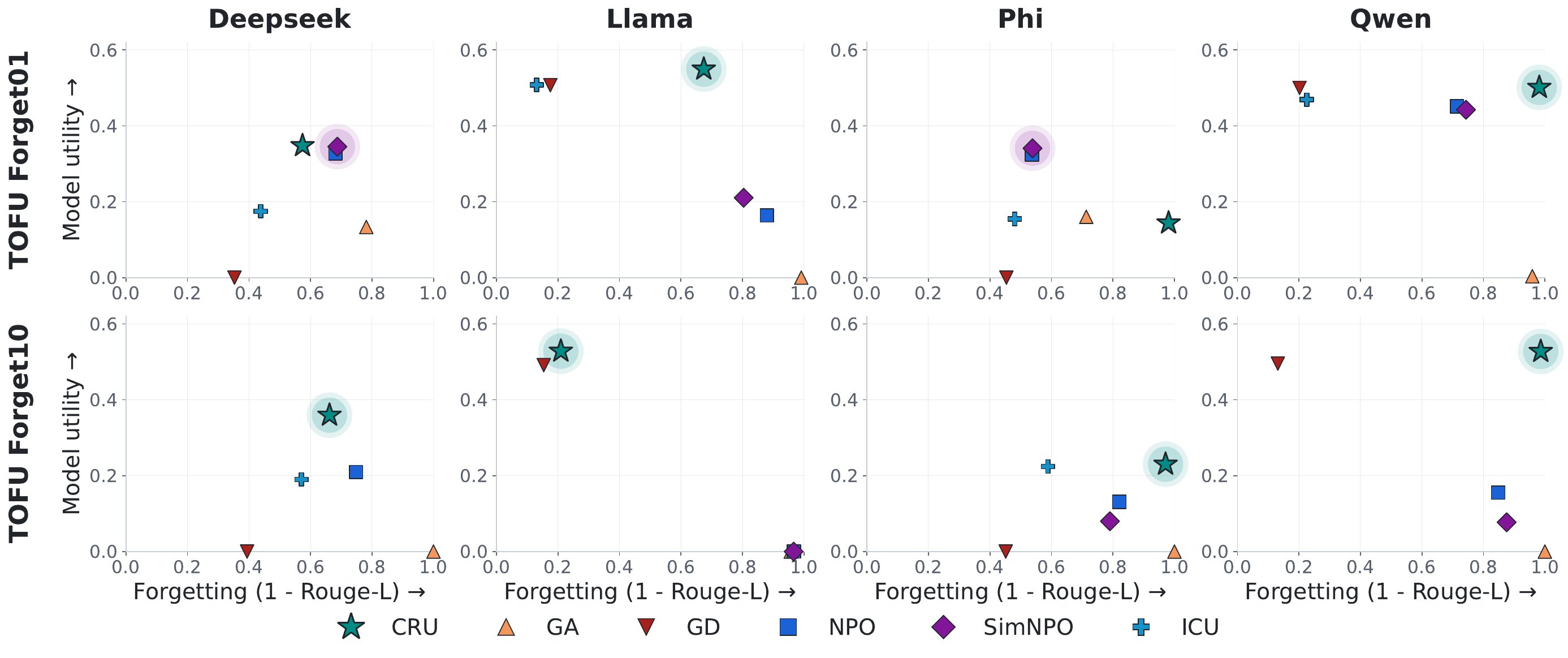}
    \caption{Model Utility and ROUGE-L Forgetting trade-off for every model and both \texttt{TOFU} splits. Higher is better for both metrics. The best method is highlighted.}
    \label{fig:pareto_grid}
\end{figure}
\begin{figure}[t]
  \centering
  \includegraphics[width=\textwidth]{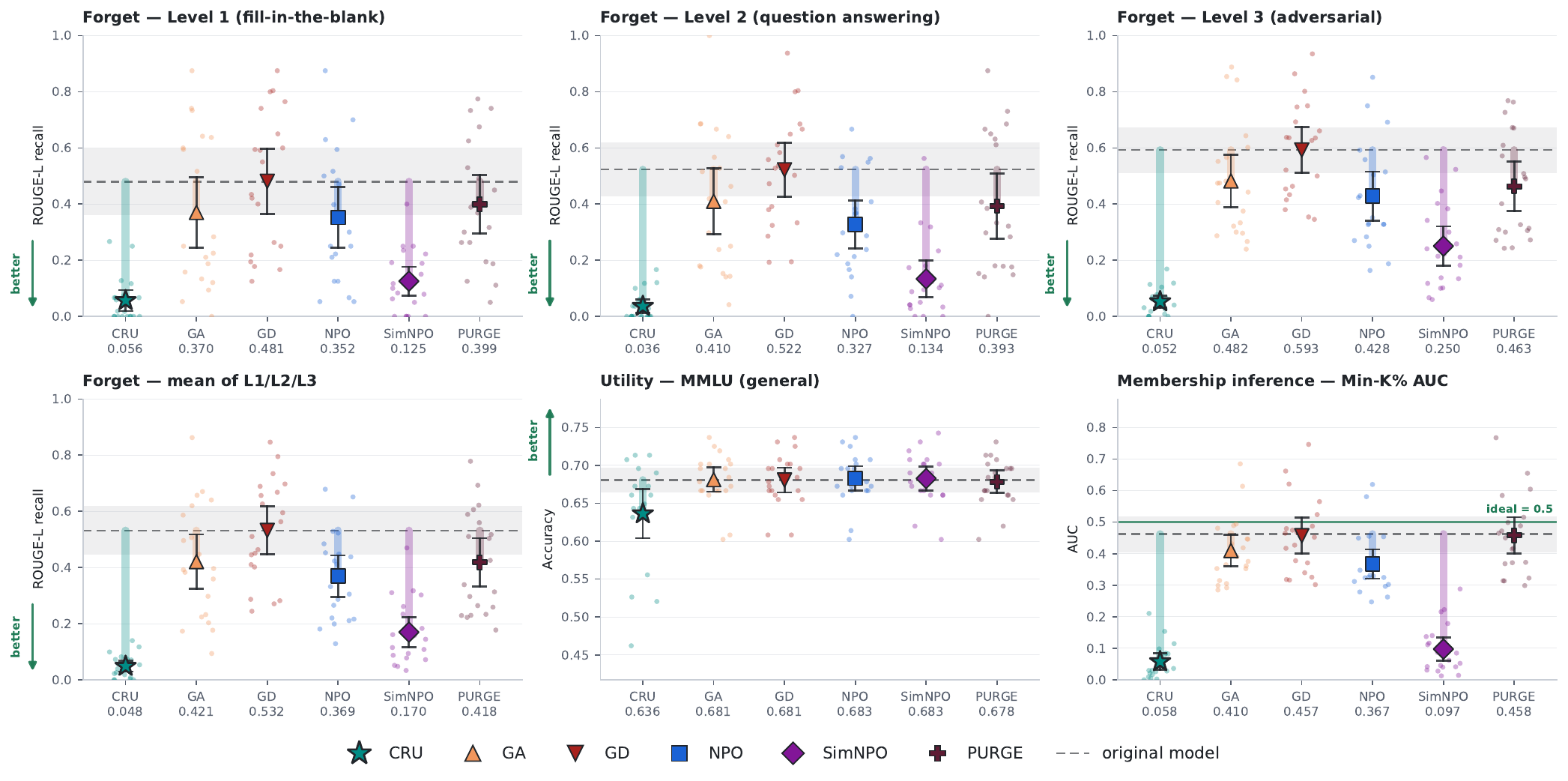}
  \caption{\textbf{\our\ on RWKU against the baselines with Phi-3-mini, 20 targets.} Each marker is the mean over targets, and the whiskers are the 95\% confidence interval; the dots behind are the individual targets, and the dashed line is the original model measured on the targets.}
  \label{fig:rwku-main}
\end{figure}

\begin{figure}[t]
  \centering
  \begin{subfigure}[c]{0.4\textwidth}
    \centering
    \includegraphics[width=\linewidth]{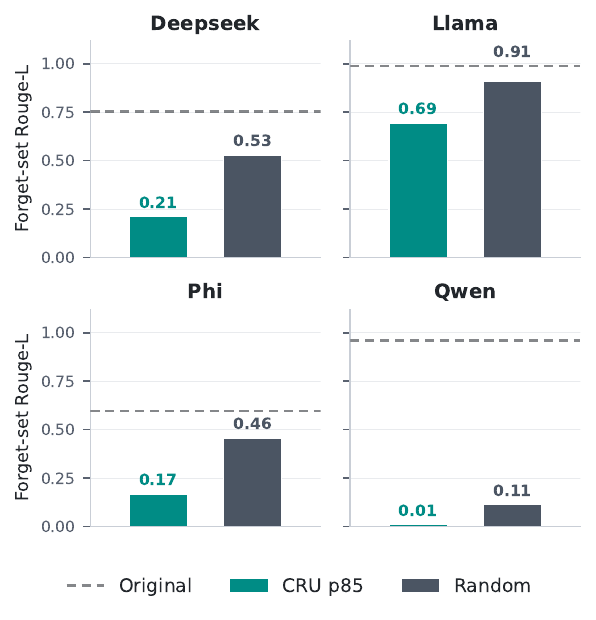}
    \caption{Rouge-L (lower is better) for random neuron selection on \texttt{TOFU} \texttt{forget01}.}
    \label{fig:random_ablation}
  \end{subfigure}
  \hfill
  \begin{subfigure}[c]{0.55\textwidth}
    \centering
    \includegraphics[width=\linewidth]{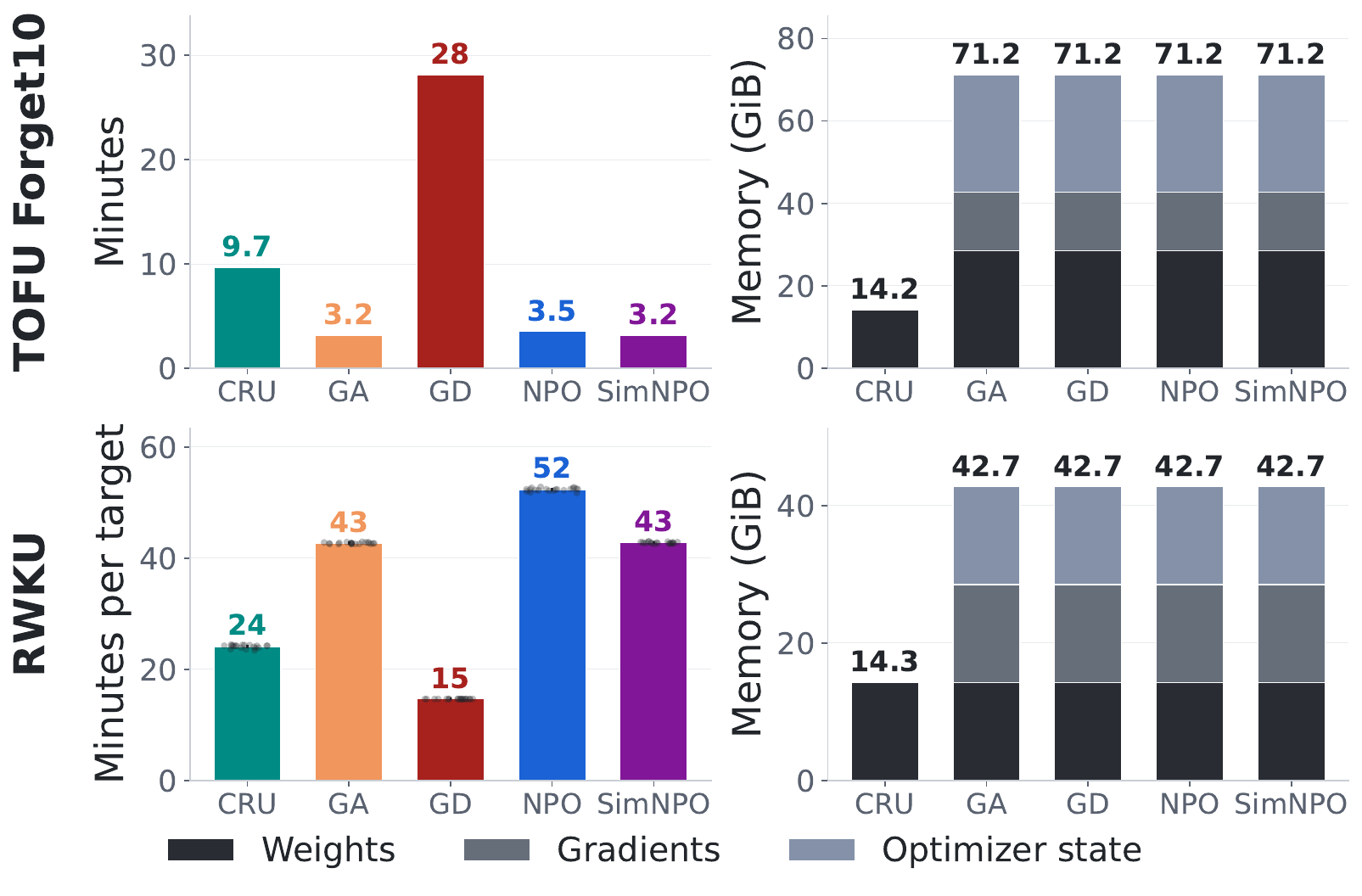}
    \caption{Resources computed on Phi with fp32 tensors.}
    \label{fig:phi_cost}
  \end{subfigure}
  \caption{Random neuron selection (left), and runtime and peak memory (right).}
  \label{fig:lunar-cost}
\end{figure}

\rqbox{\our's memory footprint is lower than all baselines. Still, \our\ has comparable runtimes to SoTA and is orders of magnitude faster than retraining.\footnotemark}
\footnotetext{According to several reports on the training of 7B+ parameter LLMs, such as the technical report of DeepSeek-V3~\citep{deepseekai2025deepseekv3technicalreport}, which lists a training time of 54 days.}

\textbf{\our\ is cheap.}
By eliminating gradients and optimizer states for the base model, \our\ significantly reduces memory overhead compared to baselines. \Cref{fig:phi_cost} reports fp32 runtime and peak memory on Phi (excluding prompt-only ICU and PURGE, evaluated via published checkpoints). \our\ demonstrates a significant memory advantage, using $14.2$ GiB on \texttt{TOFU} \texttt{forget10} and $14.3$ GiB on \texttt{RWKU}, versus $71.2$ GiB and $42.7$ GiB for all SoTA. \our\ remains competitive in runtime by requiring 24 minutes per target ($\sim$0.5 minutes for localization, the rest for routing) on \texttt{RWKU} compared to 52 minutes for NPO and 43 minutes for GA/SimNPO. On \texttt{TOFU}, CRU takes 9.7 minutes, trailing preference-based methods (3.2--3.5 minutes). We show full runtimes in \Cref{fig:full-runtime}.

\subsubsection{Ablations}
\label{sec:ablations}

\rqbox{Both phases of CRU are essential. Variance selection of Phase 1 accurately isolates concept-specific neurons. Phase 2 effectively discriminates between forget and retain sets.}

\begin{table}[h]
\centering
\begin{minipage}[c]{0.45\linewidth}
    \centering
     \includegraphics[width=\linewidth]{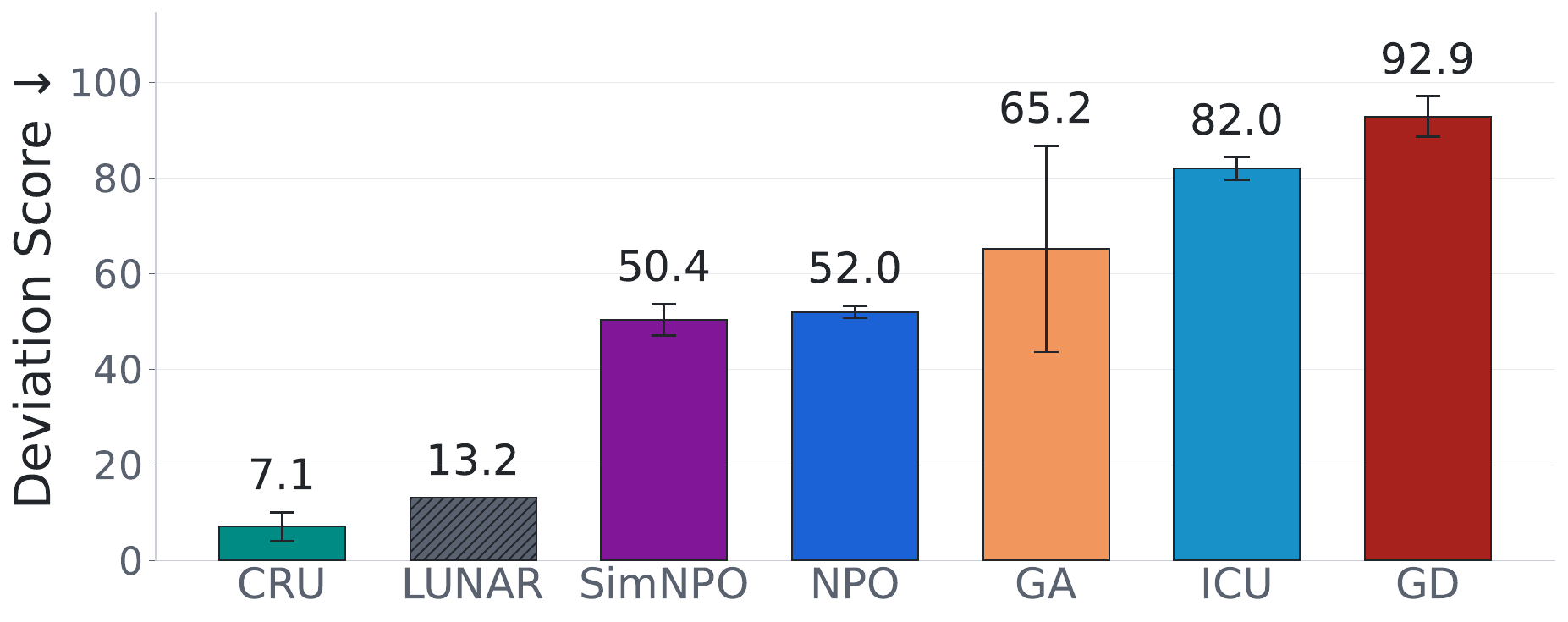}
    \captionof{figure}{ROUGE Deviation Score (lower is better) on Qwen 2.5-7B, on the \texttt{TOFU} split defined by LUNAR.}
    \label{fig:lunar_comparison}
\end{minipage}%
\hfill%
\begin{minipage}[c]{0.53\linewidth}
    \centering
    \caption{\our\ variants over $p$ on \texttt{TOFU} \texttt{forget01} (extrapolated from \Cref{tab:tofu_99_cru_ablations}). Forget is measured as $(1-\text{Rouge-L})$.}
    \label{tab:tofu_cru_ablations_summary}
    \vspace{2pt}
    \resizebox{\linewidth}{!}{%
    \begin{tabular}{lcccccccc}
    \toprule
     & \multicolumn{2}{c}{Qwen} & \multicolumn{2}{c}{Llama} & \multicolumn{2}{c}{Phi} & \multicolumn{2}{c}{DeepSeek} \\
    \cmidrule(lr){2-3} \cmidrule(lr){4-5} \cmidrule(lr){6-7} \cmidrule(lr){8-9}
    Setting & Forget$\uparrow$ & MU$\uparrow$ & Forget$\uparrow$ & MU$\uparrow$ & Forget$\uparrow$ & MU$\uparrow$ & Forget$\uparrow$ & MU$\uparrow$ \\
    \midrule
    Original & 0.04 & 0.48 & 0.01 & 0.52 & 0.07 & 0.41 & 0.25 & 0.42 \\ \midrule
    $p=85$ & \textbf{0.99} & \textbf{0.55} & 0.31 & 0.53 & 0.58 & 0.30 & \textbf{0.79} & 0.35 \\
    $p=90$ & 0.98 & 0.53 & 0.25 & 0.53 & 0.42 & 0.36 & 0.68 & \textbf{0.36} \\
    $p=95$ & 0.98 & 0.50 & \textbf{0.67} & 0.55 & 0.39 & \textbf{0.38} & 0.57 & 0.35 \\
    $p=99$ & 0.97 & 0.45 & 0.54 & \textbf{0.56} & \textbf{0.72} & 0.08 & 0.68 & 0.34 \\
    \bottomrule
    \end{tabular}%
    }
\end{minipage}
\end{table}
\textbf{The neurons we gate are concept-specific, not generically noisy.}
Because variance is partly an inherent property of individual neurons, Phase~1 might face the criticism that it merely selects units that are generically noisy across all inputs, rather than those tied to the target concept. To test this, we run Phase~1 twice on frozen weights -- i.e., once on $40$ forget-set examples ($\mathcal{C}^f_k$) and once on $40$ retain-set examples ($\mathcal{C}^r_k$) -- and compute their overlap coefficient $|\mathcal{C}^f_k \cap \mathcal{C}^r_k| / |\mathcal{C}^f_k|$ (\Cref{fig:concept-overlap}). Across gated blocks, the selected sets share only a minority of neurons: $27\%$ on Llama, $38\%$ on DeepSeek, $46\%$ on Phi, and $59\%$ on Qwen, with un-thresholded variance rankings showing similarly low-to-moderate Spearman correlations ($\rho = 0.18, 0.37, 0.44, 0.46$). Thus, the majority of gated neurons are specifically driven by the forget set. The remaining generic minority reflects baseline neuron variance, which likely accounts for the minor utility trade-offs observed on \texttt{RWKU}.

\textbf{Random neuron selection is not effective (and, rightly so).} We choose a random number of neurons in the same layers and of the same number that \our\ would have selected through variance (\Cref{fig:random_ablation}).  We compute the Forget set Rouge-L (lower is better) on \texttt{TOFU} \texttt{forget10}. Results are unambiguous: i.e., random selection with the same gating always performs worse, very close to the original model, as if no unlearning was done, across all models.

\begin{figure}[t]
  \centering
  \includegraphics[width=\textwidth]{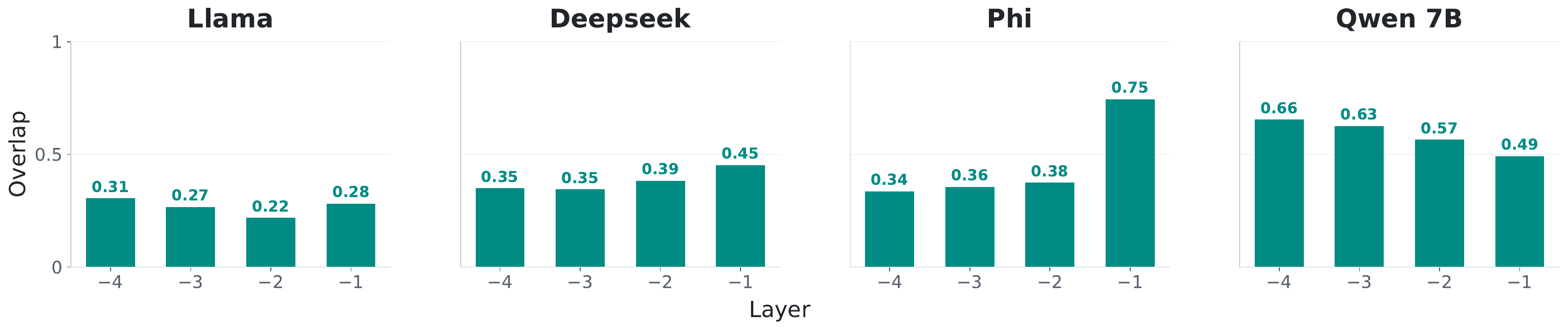}
  \caption{\textbf{Neuron overlap for the forget and retain sets.} Each bar is the fraction of the neurons in $C_k^{f}$ that are also in $C_k^{r}$, $|C_k^{f} \cap C_k^{r}| / |C_k^{f}|$: how many of the gated neurons would have been picked without ever showing the model the forget set. Compued on \texttt{TOFU Forget01} with $p=95$.}
  \label{fig:concept-overlap}
\end{figure}

\textbf{Gating fewer neurons does not produce a gentler intervention.} 
Intuitively, adjusting the percentile threshold $p$ to restrict $|\mathcal{C}_l|$ should yield a monotonic trade-off between forgetting and utility. However, on \texttt{RWKU}, $p=90$ achieves the best performance on both axes (forget $0.055$, utility $0.418$), whereas $p=99$, which gates only a fifth as many neurons, degrades both metrics (forget $0.101$, utility $0.354$). This occurs because the router compensates for the restricted set of neurons by closing the remaining gates more aggressively. The same non-monotonic pattern holds on \texttt{TOFU}, where Phi's utility drops to $0.08$ at $p=99$ compared to $0.38$ at $p=95$ (\Cref{tab:tofu_cru_ablations_summary}; see also \Cref{fig:cru_percentile_ablation,tab:tofu_99_cru_ablations}). In short, restricting where suppression can land concentrates the intervention.

\section{Conclusion}\label{sec:conclusion}

\our\ is a streamlined and lightweight unlearning method that applies seamlessly to \emph{any} LLM. One untrained forward pass over the forget set ranks neurons by how much their pooled activations vary, and small routing modules, then gate those neurons per token, closing on forget inputs and passing on retain ones. The base model is frozen throughout, so every change in behavior is attributable to the gated neurons rather than to parameter drift.

Our analysis shows three advantages. First, high forgetting scores are not evidence of unlearning: on \texttt{TOFU} \texttt{forget10}, GA, NPO, and SimNPO report up to $0.97$ forgetting at null utility, which is a broken model reported as a success, and any comparison that does not read the two axes together will reward that. Second, suppressing a surface form is not the same as removing what produces it. Every baseline we tested loses ground as the probe gets harder, while \our\ holds at $0.052$ adversarial recall, compared to $0.250$ for the next-best, which reflects the difference between knowledge that is gone and knowledge that is merely inconvenient to reach. Third, the cost of unlearning is mostly the cost of retraining: dropping gradients and optimizer states for the base model reduces the footprint from 71 GiB to 14 GiB at comparable runtime, putting per-concept unlearning within reach of a single consumer device.

\textbf{Limitations.} \our\ leaves a few things open. Phase 1 uses variance as a cheap stand-in for attribution, and only a minority of the neurons it selects are generically variable rather than concept-specific. The gates withhold some activation on retain inputs, which accounts for the small utility drop we observe on \texttt{RWKU}. Our choice of the last four blocks follows the localization literature rather than our own criterion. Because each request trains its own routers on a frozen base, sequential and simultaneous unlearning of multiple concepts is a direction we leave for future work. We also leave the ablation of \Cref{eq:objective} to support the choice of the two terms as a future avenue  (\S\ref{app:loss-shape}).

\cta{We hope the broader point outlives our method. Most unlearning research has asked what to optimize and answered by changing the weights. We suggest that concepts can be reached via the neurons that express them, and that suppression can be determined per query rather than baked into the parameters. We encourage future work in unlearning to treat the scope of the intervention as a design choice in its own right.}

 
\subsection*{AI use statement}

In this work, we used generative AI tools to formalize mathematical claims and to aid in implementing methods. 

We have \textbf{not} used generative AI tools to assist with translation, develop theoretical models, provide critical ingredients for proving mathematical claims, provide feedback on research methodology, support qualitative and thematic data analysis, or interpret results. Finally, generating synthetic datasets, assisting in writing proofs, proposing hypotheses, and cleaning datasets are not applicable to this work. 

Additionally, we used generative AI tools to create scientific figures and edit software code. We have reviewed all AI-assisted work: all co-authors have checked and revisited the text, including the mathematical formulations; LLM-generated code was verified and tested for correctness by 2 authors; and all figures were reviewed. 

We take responsibility for the final content of this work, including text, claims, or artifacts produced with the aid of generative AI.

\subsection*{Reproducibility statement}

We made every effort to ensure the results were reproducible. Our repository is publicly available at \url{https://anonymous.4open.science/r/cru-iclr-4E56/README.md} and includes a detailed README with instructions for reproducing the experiments.






\bibliography{ref}
\bibliographystyle{iclr2027_conference}

\appendix

\section{What SoTA does with Equations~\eqref{eq:unlearning-objective} and~\eqref{eq:scope}}
\label{app:sota}

\subsection{The objective}
\label{app:objective}

Recall~\Cref{eq:unlearning-objective}. GA sets $\lossf = -\ell$ with $\lambda =
0$~\citep{yao2024large}; GD keeps $\lossf = -\ell$ and
takes $\lambda > 0$~\citep{liu2022continual}. Preference-based objectives replace $-\ell$, which is unbounded below, with a saturating alternative. Writing $\sigma$ for the logistic function and $y_{\mathrm{idk}}$ for a refusal response, the DPO formulation prefers refusal over the true
completion~\citep{rafailov2023direct},
\begin{equation}\label{eq:dpo}
    \lossf^{\mathrm{DPO}}(y \mid x; \theta)
    = -\log \sigma\!\left(
        \beta \log \frac{\pth(y_{\mathrm{idk}} \mid x)}{\porig(y_{\mathrm{idk}} \mid x)}
        - \beta \log \frac{\pth(y \mid x)}{\porig(y \mid x)}
      \right),
\end{equation}
whereas NPO retains only the dispreferred branch~\citep{npo},
\begin{equation}\label{eq:npo}
    \lossf^{\mathrm{NPO}}(y \mid x; \theta)
    = -\frac{2}{\beta} \log \sigma\!\left(
        -\beta \log \frac{\pth(y \mid x)}{\porig(y \mid x)} \right)
    = \frac{2}{\beta}\,
      \log\!\left(1 + \left[\frac{\pth(y \mid x)}{\porig(y \mid x)}\right]^{\beta}\right),
\end{equation}
which recovers GA ascent as $\beta \to 0$ but, being bounded below by zero, diverges exponentially more slowly. SimNPO discards the reference model and normalizes by response length~\citep{simnpo},
\begin{equation}\label{eq:simnpo}
    \lossf^{\mathrm{SimNPO}}(y \mid x; \theta)
    = -\frac{2}{\beta} \log \sigma\!\left(
        -\frac{\beta}{|y|} \log \pth(y \mid x) \right).
\end{equation}
Throughout, $\beta > 0$ is the inverse temperature of \Cref{eq:dpo,eq:npo,eq:simnpo} and $\lambda$ the forget/retain weight of \Cref{eq:unlearning-objective}.
 
The second template replaces likelihood targets with a scalar reward on sampled completions. With a reward $r(x,y)$ that penalizes any mention of the forbidden concept, unlearning becomes a KL-constrained policy-optimization problem~\citep{zaradoukas2026reinforcement},
\begin{equation}\label{eq:rl}
    \max_{\theta \in \Theta} \;\;
    \mathbb{E}_{x \sim \Df} \, \mathbb{E}_{y \sim \pth(\cdot \mid x)}
        \!\left[ r(x,y) \right]
    \;-\; \lambda_{\mathrm{KL}} \,
    \mathbb{E}_{x \sim \Df} \!\left[
        \mathrm{KL}\!\left( \pth(\cdot \mid x) \,\|\, \porig(\cdot \mid x) \right)
    \right],
\end{equation}
where the KL term plays the utility-preserving role that
$\mathbb{E}_{\Dr}[\lossr]$ plays in \Cref{eq:unlearning-objective}, and advantages are estimated group-relatively over completions sampled per prompt rather than by a learned critic.
 
The third template does not optimize at all. Prompt-based methods leave the weights fixed and act on the context: for a context $c$,
\begin{equation}\label{eq:icu}
    \theta_{\mathrm{u}} = \theta_{\mathrm{o}},
    \qquad
    \punl(y \mid x) \;=\; \porig(y \mid c \oplus x),
\end{equation}
with $\oplus$ denoting concatenation~\citep{pawelczyk2023context}. The consequence is visible in the notation: since $\theta_{\mathrm{u}} = \theta_{\mathrm{o}}$, forgetting is a property of the session rather than of the model, and any measurement taken with $c$ removed returns $\porig$ exactly.
 
\subsection{The scope, and where CRU sits}
\label{app:scope}

\paragraph{What ${\mathbf{w}}$ is allowed to hold.} ~\Cref{eq:scope} writes a parameter-space unlearning operator as $\theta_u = \theta_o + \mathbf{w} \odot \Delta$, and its two factors divide the labor: $\Delta$ is where the objective of \S\ref{app:objective} ends up, and $\mathbf{\mathbf{w}}$ is everything a method has to say about scope. Read this way, the literature settles $\mathbf{\mathbf{w}}$ in one of two ways. Global methods (GA, GD, NPO, SimNPO, PURGE) fix $\mathbf{\mathbf{w}} = \mathbf{1}$ and leave scope to the objective. Localized ones (SalUn, WAGLE, CLUE, LUNAR) choose $\mathbf{\mathbf{w}} \neq \mathbf{1}$ based on saliency, influence attribution, circuit discovery, or layer selection. In both cases, $\mathbf{\mathbf{w}}$ is a constant vector, computed once while unlearning and applied unchanged to every query that follows.~\Cref{eq:scope} has no slot in which to write a quantity that depends on the input, so a method that wants one cannot be written in it at all.

\paragraph{\our\ leaves the equation empty.} \our\ sets $\mathbf{w} \odot \Delta = \mathbf{0}$, so that $\theta_u = \theta_o$ exactly, and every weight of the base model survives unlearning bit for bit. What replaces $\mathbf{w}$ is a map from a hidden state to a full-width multiplier. For block $l$, with concept set $\mathcal{C}_l = \{j_1 < \cdots < j_{|\mathcal{C}_l|}\}$ and router parameters $\phi_l = (\mathbf{\mathbf{w}}_1, \mathbf{b}_1, \mathbf{\mathbf{w}}_2, \mathbf{b}_2)$, define $U_l : \mathbb{R}^d \to [0,1]^d$ by
\begin{equation}
    U_l(\mathbf{h})[j] \;=\;
    \begin{cases}
        R_l(\mathbf{h})[m], & j = j_m \in \mathcal{C}_l, \\[2pt]
        1, & j \notin \mathcal{C}_l,
    \end{cases}
    \qquad
    R_l(\mathbf{h}) = \sigma\!\left(\mathbf{\mathbf{w}}_2 \operatorname{ReLU}(\mathbf{\mathbf{w}}_1 \mathbf{h}
    + \mathbf{b}_1) + \mathbf{b}_2\right),
    \label{eq:gate-map}
\end{equation}
so that~\Cref{eq:gating} is the product
\begin{equation}
    \mathbf{h}'_{l,t} \;=\; U_l(\mathbf{h}_{l,t}) \odot \mathbf{h}_{l,t}.
    \label{eq:scope-cru}
\end{equation}
$U_l$ is determined by two things of different kinds. $\mathcal{C}_l$ is a constant of the map, fixed once per concept in Phase~1 and never differentiated; $\phi_l$ is trained. Coordinates outside $\mathcal{C}_l$ are pinned to $1$ by construction, which is what makes the set a hard restriction on where suppression may land. The map carries no dependence on $t$: the same $U_l$ is applied at every position.

\Cref{eq:scope-cru} is~\Cref{eq:scope} with three substitutions. The operand is an activation rather than a parameter vector; the product runs over model width at a single position rather than over $\dim\Theta$; and the multiplier is the image of a function rather than a stored vector. \textbf{Only the third is the claim of this paper.} The first two say where we act, and one could imagine acting there with a constant; the third says that what we do once we are there is not fixed until the query arrives.

\paragraph{A fixed selector and a learned one.} \our\ does not abolish the constant selector of~\Cref{eq:scope} so much as demote it. The set $\mathcal{C}_l$ is computed once per concept and is as fixed thereafter as any $\mathbf{w}$; what changes is its authority. It decides where suppression is permitted to land, never how much of it lands, and a neuron inside $\mathcal{C}_l$ is only a candidate. The magnitude is settled per token by $R_l$, which is why the same neuron can be closed at one position and passed at the next (\S\ref{sec:phase2}).

\paragraph{Special cases.} Because $U_l$ is a function, the frozen model is recovered by restricting the class it is drawn from, and the restrictions form an order.
\begin{itemize}[leftmargin=1.5em, itemsep=1pt, topsep=2pt]
    \item $U_l \equiv \mathbf{1}_d$ returns the base model.
    \item $U_l$ constant in $\mathbf{h}$, with entries in $\{0,1\}$, is a fixed hard mask over a neuron subset, which is basically NeuMuter.
    \item $U_l$ piecewise constant, taking $\mathbf{1}_d$ or a tuned magnitude $\alpha$ according to whether a similarity $s(\mathbf{h})$ clears a grid-searched threshold $\tau$, gives CAST and DSG.
    \item $U_l$ smooth and learned, as in ~\Cref{eq:gate-map} (i.e.,~\Cref{eq:gating}), is \our.
\end{itemize}
The ordering is by how much of $U_l$ is decided before the model is ever queried. The rule-based methods fix the magnitude and condition the trigger; ~\Cref{eq:gating,eq:gate-map} learns both.

\paragraph{What this costs and returns.} Since $\Delta$ is never formed, the unlearning operator is $\mathcal{U}(\pi_o, \mathcal{D}_f, \mathcal{D}_r) = (\theta_o, \phi)$. No gradient or optimizer state is ever allocated for $\theta_o$, which is the whole of the memory result in \S\ref{sec:results} rather than an implementation detail. The operator also has an inverse, whose methods that write into $\theta$ do not. In other words, deleting $\phi$ returns $\pi_o$, not an approximation of it.

\paragraph{Why this is not ICU.} Prompt-based unlearning also leaves $\theta_u = \theta_o$. Nevertheless, per~\Cref{eq:icu}, the intervention lives in the context $c$, so forgetting is a property of the session. In~\Cref{eq:scope-cru}, the intervention lives in $\phi$, which travels with the weights and not the input.

\section{On the shape of the two loss terms}
\label{app:loss-shape}

\Cref{eq:objective} pairs a linear forget term with a quadratic retain term. The asymmetry is not forced, so this section gives the reasons for it: what each shape does to the gradient reaching the router (\S\ref{app:shape-grad}), and what the four combinations predict for a router that cannot separate the two corpora (\S\ref{app:shape-collapse}).

\subsection{Gradients}
\label{app:shape-grad}

Write $z_n$ for the pre-activation of the router's output entry that produces gate $w_n =\sigma(z_n)$, and let $\bar{w}_l = \tfrac{1}{N}\sum_{n=1}^N w_n$ be the block mean over the $N$ selected-neuron, non-padding entries of a batch. Every shape we consider reaches $z_n$ through the same Jacobian,
\begin{equation}
    \frac{\partial w_n}{\partial z_n} \;=\; w_n (1 - w_n),
    \label{eq:app-sigmoid-jac}
\end{equation}
so the shapes differ only in the scalar multiplying it. For the forget term,
\begin{equation}
    \frac{\partial\, \bar{w}_l}{\partial z_n} = \frac{1}{N}\, w_n(1-w_n),
    \qquad
    \frac{\partial\, \bar{w}_l^{\,2}}{\partial z_n} = 2\bar{w}_l \cdot \frac{1}{N}\, w_n(1-w_n),
    \label{eq:app-forget-grad}
\end{equation}
and for the retain term, with target $1$,
\begin{equation}
    \frac{\partial\, (1 - \bar{w}_l)}{\partial z_n} = -\frac{1}{N}\, w_n(1-w_n),
    \qquad
    \frac{\partial\, (\bar{w}_l - 1)^2}{\partial z_n}
        = 2(\bar{w}_l - 1) \cdot \frac{1}{N}\, w_n(1-w_n).
    \label{eq:app-retain-grad}
\end{equation}

\paragraph{The retain term should vanish at its target; the forget term should not.} The two terms are not two instances of the same kind of goal. $\mathcal{L}_{\text{suppress}}$ has no target. In other words, a closed gate is better than a half-closed one, and there is no value of $\bar{w}^{\,f}_l$ at which the method would prefer to stop. $\mathcal{L}_{\text{pass}}$ has one, i.e., $\bar{w}^{\,r}_l = 1$, at which the router is passing retain activations unchanged. A penalty on a constraint should therefore release its gradient at the constraint, which \Cref{eq:app-retain-grad} does in the quadratic case and does not in the linear one: a linear retain term pushes toward $1$ with undiminished force at $\bar{w}^{\,r}_l = 0.999$, so the equilibrium between the two terms would be set by their relative weight alone rather than by how far the retain gates have actually drifted. The quadratic makes retain pressure a function of the violation, which is the standard reason for a squared penalty and the reason here.

\paragraph{Why the forget term is not squared.} The same reasoning does not transfer to $\mathcal{L}_{\text{suppress}}$, because the router already saturates on its own. The factor $w_n(1-w_n)$ of \Cref{eq:app-sigmoid-jac} vanishes as a gate shuts, so the gradient reaching a nearly closed gate is small before any choice of loss shape. Squaring the forget term rescales it by a further $2\bar{w}^{\,f}_l$, a second factor that vanishes with the first: at the block means we observe on \texttt{TOFU} \texttt{forget10} the rescaling is $2\bar{w}^{\,f}_l \in [0.02, 0.12]$ (\Cref{fig:gate_gap}), between roughly $8\times$ and $50\times$ smaller than the linear term depending on the model, and shrinking further as the gate approaches zero. The linear term holds the multiplier at $1$ and leaves only the sigmoid's own saturation. Linearity is affordable here in a way it is not for gradient ascent (\S\ref{app:objective}): since $\bar{w}^{\,f}_l \in [0,1]$, $\mathcal{L}_{\text{suppress}}$ is bounded below by construction.

\paragraph{Consequences for what the gates settle at.} Because $\mathcal{L}_{\text{pass}}$ relaxes near $1$ while $\mathcal{L}_{\text{suppress}}$ does not relax near $0$, and the two act on the same router parameters over different batches, the retain gates come to rest short of their target: \Cref{fig:gate_gap} reports $0.79$--$0.85$ rather than $1$. This is the intended trade and not a training failure. The withheld activation is bounded by construction, since the base weights are untouched and the gate lies in $[0,1]$, so its cost appears as the small utility drop in \texttt{RWKU}, rather than as parameter drift. A linear retain term would hold the gates closer to $1$, at the cost of making that pressure insensitive to how much was withheld.

\paragraph{A note on one-sidedness.} Since $w_n \in [0,1]$ we have $\bar{w}^{\,r}_l \le 1$ throughout, so $(\bar{w}^{\,r}_l - 1)^2 = (1 - \bar{w}^{\,r}_l)^2$ and the square never penalises overshoot. It is a one-sided penalty in effect, and the squared form is chosen for the vanishing gradient at the target rather than for two-sidedness.

\subsection{What each combination predicts for a collapsed router}
\label{app:shape-collapse}

\Cref{prop:tied} treats the hypothesis that the router cannot distinguish the two corpora, so that a single gate value $w$ is returned on both. The prediction it makes depends on the shapes, and not every pairing makes one.

\begin{proposition}
\label{prop:shape-collapse}
Let $\mathcal{L}(w) = f(w) + \lambda\, r(w)$ be the collapsed loss under a constant gate $w$, with $\lambda > 0$. Then:
\begin{enumerate}[leftmargin=1.5em, itemsep=1pt, topsep=2pt]
    \item $f(w) = w$, $r(w) = (w-1)^2$: $\;\mathcal{L}$ is strictly convex and minimized at $w^\star = 1 - \tfrac{1}{2\lambda}$, which lies in $(0,1)$ iff $\lambda > \tfrac{1}{2}$.
    \item $f(w) = w^2$, $r(w) = (w-1)^2$: $\;w^\star = \tfrac{\lambda}{1+\lambda} \in (0,1)$ for every $\lambda > 0$.
    \item $f(w) = w$, $r(w) = 1-w$: $\;\mathcal{L}(w) = (1-\lambda)w + \lambda$ is affine, so the minimum over $[0,1]$ is attained at an endpoint according to the sign of $1-\lambda$, and every $w$ is optimal at $\lambda = 1$.
    \item $f(w) = w^2$, $r(w) = 1-w$: $\;w^\star = \tfrac{\lambda}{2}$, which lies in $(0,1)$ iff $\lambda < 2$.
\end{enumerate}
\end{proposition}
\begin{proof}
Each $\mathcal{L}$ is an affine or quadratic polynomial in $w$; set $\mathcal{L}'(w) = 0$ and check the sign of $\mathcal{L}''$. In case 3, $\mathcal{L}'' = 0$ and the minimum is attained at an endpoint of $[0,1]$.
\end{proof}

Case 3 is the reason we would not use a linear retain term, even setting \S\ref{app:shape-grad} aside. With both terms linear, the collapsed loss is affine, its optimum sits at a boundary of the box, and at $\lambda = 1$ the loss is constant in $w$: the null model predicts nothing, and a measurement of the gate means cannot be compared against it. Case~1, the pairing we use, gives the null an interior value ($0.5$ at $\lambda = 1$) that our measurements contradict by more than an order of magnitude on the forget side. Case~2 also gives an interior value, but at the cost of the gradient rescaling of \Cref{eq:app-forget-grad}. Case~4 makes the location of the null grow linearly in $\lambda$ and leaves $[0,1]$ at $\lambda = 2$, so for the retained weights, we use it to predict a boundary again.

\paragraph{What we have and have not shown.} The argument above is analytical. It says what each pairing does to the gradient reaching the router and what each predicts for a collapsed router, and on both counts the linear-quadratic pairing is the one we can reason about: the forget term avoids compounding the sigmoid's own saturation, the retain term releases its gradient at the constraint it encodes, and the null model of \Cref{prop:tied} lands at an interior value our measurements contradict. We have not run the corresponding ablation, so we do not claim that the alternatives fail empirically, nor that \our's results depend on this choice. The gate statistics we report \Cref{fig:gate_gap}) are consistent with the mechanism described here, but do not discriminate among the four shapes; a sweep over them is the natural next check.
\section{Detailed Experimental Settings}\label{apx:full_exp_settings}

\paragraph{Hardware.} We carry out all our experiments on a single NVIDIA GB10 Grace Blackwell system equipped with a 20-core Arm CPU and 128 GB of unified memory shared between the CPU and the integrated Blackwell GPU.

\paragraph{Models.}  \Cref{tab:models} shows the family, number of parameters, number of layers, and context (in terms of thousands of tokens) of the models we employ in our experimental settings.

\paragraph{Benchmark metrics.} \Cref{tab:benchmark_metrics} shows all the metrics evaluated for each dataset. For \texttt{TOFU}, utility metrics are evaluated for three sets: the retain set, i.e., the training set without the forget set, and two held-out sets called ``Real Authors'' and ``World Facts'' \citep{maini2024tofu}. These are not part of the model's fine-tuning and contain information assumed to be already known.

\begin{table}[ht]
\centering
\small
\setlength{\tabcolsep}{5pt}
\caption{The four model families we evaluate on. Sizes are the total parameter counts of
the released checkpoints.}
\label{tab:models}
\begin{tabular}{@{}llrrr@{}}
\toprule
Short name & Family  & Params & Layers & Context \\
\midrule
DeepSeek & DeepSeek LLM   & 6.9B & 30 & 4K   \\
Qwen     & Qwen2.5                            & 7.6B & 28 & 32K  \\
Llama    & Llama 3.1         & 8.0B & 32 & 128K \\
Phi      & Phi-3        & 3.8B & 32 & 4K   \\
\bottomrule
\end{tabular}
\end{table}

\begin{table}[ht]
\centering
\footnotesize
\setlength{\tabcolsep}{4pt}
\caption{Evaluation metrics of the two benchmarks.}
\label{tab:benchmark_metrics}
\begin{tabular}{@{}p{0.22\textwidth}p{0.72\textwidth}@{}}
\toprule
Metric & What it measures \\
\midrule
\multicolumn{2}{@{}l}{\texttt{TOFU}} \\
\addlinespace[1pt]
Probability & Normalized likelihood $P(a\mid q)^{1/|a|}$ of the ground-truth answer (multiple-choice variant on Real Authors / World Facts). Forget $\downarrow$, utility $\uparrow$. \\
ROUGE-L recall & Overlap between the generated answer and the reference answer. Forget $\downarrow$, utility $\uparrow$. \\
Truth Ratio & Likelihood of perturbed (false) answers over that of a correct paraphrase: how much the model now prefers a wrong answer. Forget $\uparrow$, utility $\downarrow$. \\
Forget Quality $\uparrow$ & KS-test $p$-value against a model retrained on the retain set only: high $p$ = indistinguishable from never having seen the forgotten data. \\
Model Utility $\uparrow$ & Harmonic mean of the nine utility scores (probability, ROUGE-L, rescaled truth ratio $\times$ retain / Real Authors / World Facts). \\
\addlinespace[2pt]
\midrule
\multicolumn{2}{@{}l}{\texttt{RWKU}} \\
\addlinespace[1pt]
Forget FB / QA / AA $\downarrow$ & ROUGE-L recall on cloze, question, and adversarial (jailbreak-style) probes about the target; AA exposes knowledge that is only suppressed at the surface. \\
Neighbor FB / QA $\uparrow$ & The same probes on closely related entities that must be kept: edit locality. \\
MIA (LOSS, Zlib, Min-K\%++) & Membership-inference scores on held-out fragments about the target (FM $\uparrow$) and about unrelated subjects (RM, unchanged); the AUC separating the two should return to chance ($\to 0.5$). \\
Utility $\uparrow$ & MMLU 5-shot accuracy (general), BBH 3-shot CoT exact match (reasoning), TruthfulQA MC1 (truthfulness), TriviaQA F1 (factuality) and AlpacaEval $n$-gram entropy (fluency, i.e.\ non-degenerate text). \\
\bottomrule
\end{tabular}
\end{table}

\paragraph{Comparison with LUNAR.} The authors of LUNAR, in their original paper \citep{lunar}, use a custom split of TOFU that is not part of the original benchmark. Specifically, they use a forget set consisting of a single author; we conformed to this setting by creating an equal custom split on which we evaluate their own metric, \textit{deviation score}. The deviation score is defined as: 
\[
  \mathrm{DS} = 100 \times \sqrt{\mathrm{ROUGE1}_{\mathrm{forget}}^{2}
                                 + \bigl(1 - \mathrm{ROUGE1}_{\mathrm{retain}}\bigr)^{2}}.
\]
That is the Euclidean distance between a method's operating point and the ideal outcome of unlearning, at which the model reproduces none of the forget set ($\mathrm{ROUGE1}_{\mathrm{forget}} = 0$) while remaining unchanged on the retain set ($\mathrm{ROUGE1}_{\mathrm{retain}} = 1$). It is therefore a quantity to be minimized, with $0$ denoting perfect unlearning and a fine-tuned model that has not been unlearned scoring close to $100$. Following LUNAR, we score ROUGE-1 recall rather than the ROUGE-L recall used elsewhere in this work. We report the result on Qwen2.5-7B, the only model common to both papers.

\paragraph{Comparison with PURGE.} For \texttt{RWKU}, we adapt our settings to the ones of PURGE \citep{zaradoukas2026reinforcement} for a faithful comparison: the forget set is composed of one author at a time, and we report metrics that are the mean of 20 authors. Moreover, we use the publicly released model checkpoints from the original paper; their download is integrated into our evaluation code.

\subsection{Hyperparameters}\label{apx:hyperparameters}

\begin{figure}
    \centering
    \includegraphics[width=1\linewidth]{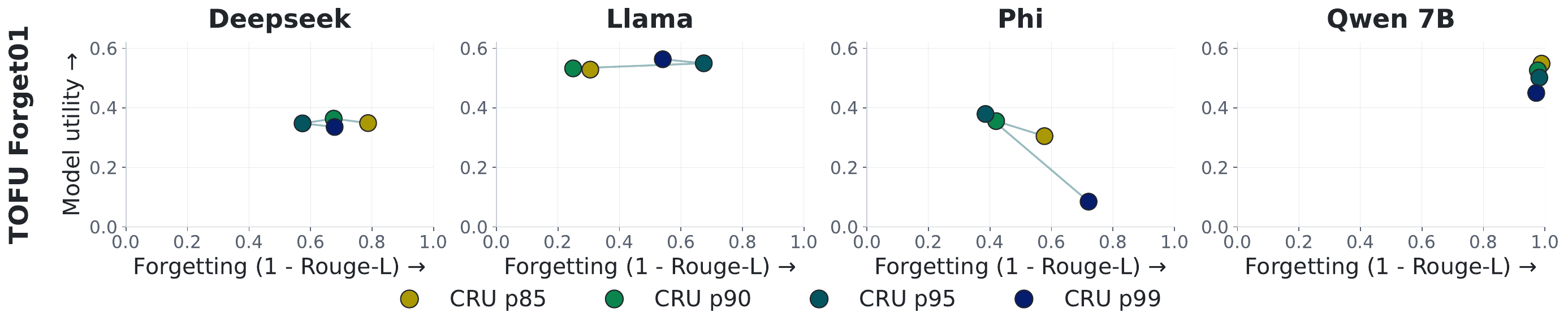}
    \caption{Hyperparameter $p$ study for \our\ on \texttt{TOFU}.}
    \label{fig:cru_percentile_ablation}
\end{figure}

\begin{table}[t]
\centering
\small
\caption{CRU percentile ablation on \texttt{TOFU} \texttt{forget01}. ``Original'' is the model before unlearning. Best method per column, within each model, in bold.}
\label{tab:tofu_99_cru_ablations}
\resizebox{\textwidth}{!}{%
\begin{tabular}{llccccccccccccc}
\toprule
 &  & \multicolumn{3}{c}{Unlearning Efficacy (Forget Set)} & \multicolumn{10}{c}{Utility Preservation} \\
\cmidrule(lr){3-5}\cmidrule(lr){6-15}
 &  & \multicolumn{3}{c}{} & \multicolumn{3}{c}{Real Authors} & \multicolumn{3}{c}{World Facts} & \multicolumn{3}{c}{Retain Set} &  \\
\cmidrule(lr){6-8}\cmidrule(lr){9-11}\cmidrule(lr){12-14}
Model & Method & (1-Rouge-L)$\uparrow$ & (1-Prob.)$\uparrow$ & Truth ratio$\downarrow$ & Rouge-L$\uparrow$ & Prob.$\uparrow$ & Truth ratio$\uparrow$ & Rouge-L$\uparrow$ & Prob.$\uparrow$ & Truth ratio$\uparrow$ & Rouge-L$\uparrow$ & Prob.$\uparrow$ & Truth ratio$\uparrow$ & MU$\uparrow$ \\
\midrule
\multirow{5}{*}{Qwen} & Original & 0.04 & 0.01 & -- & 0.83 & 0.31 & 0.32 & 0.94 & 0.37 & 0.29 & 0.98 & 0.99 & 0.59 & 0.48 \\
 & CRU p85 & \textbf{0.99} & \textbf{1.00} & \textbf{0.23} & 0.63 & \textbf{0.38} & \textbf{0.53} & 0.85 & \textbf{0.39} & 0.39 & 0.83 & 0.91 & 0.58 & \textbf{0.55} \\
 & CRU p90 & 0.98 & 1.00 & 0.34 & 0.74 & 0.36 & 0.44 & \textbf{0.88} & 0.37 & 0.32 & \textbf{0.94} & \textbf{0.96} & \textbf{0.60} & 0.53 \\
 & CRU p95 & 0.98 & 0.99 & 0.33 & \textbf{0.77} & 0.32 & 0.37 & 0.84 & 0.37 & 0.33 & 0.94 & 0.95 & 0.56 & 0.50 \\
 & CRU p99 & 0.97 & 0.92 & 0.37 & 0.41 & 0.37 & 0.49 & 0.52 & 0.37 & \textbf{0.39} & 0.51 & 0.53 & 0.58 & 0.45 \\
\midrule
\multirow{5}{*}{Llama} & Original & 0.01 & 0.01 & -- & 0.76 & 0.33 & 0.41 & 0.88 & 0.39 & 0.34 & 0.98 & 0.99 & 0.56 & 0.52 \\
 & CRU p85 & 0.31 & 0.23 & 0.49 & 0.74 & 0.34 & 0.44 & 0.86 & 0.38 & 0.35 & \textbf{0.98} & \textbf{0.99} & \textbf{0.55} & 0.53 \\
 & CRU p90 & 0.25 & 0.18 & 0.48 & \textbf{0.77} & 0.35 & 0.45 & \textbf{0.88} & 0.38 & 0.35 & 0.98 & 0.98 & 0.55 & 0.53 \\
 & CRU p95 & \textbf{0.67} & \textbf{0.79} & \textbf{0.47} & 0.53 & \textbf{0.49} & \textbf{0.68} & 0.79 & \textbf{0.41} & \textbf{0.44} & 0.59 & 0.66 & 0.55 & 0.55 \\
 & CRU p99 & 0.54 & 0.54 & 0.49 & 0.62 & 0.45 & 0.64 & 0.81 & 0.40 & 0.41 & 0.72 & 0.81 & 0.55 & \textbf{0.56} \\
\midrule
\multirow{5}{*}{Phi} & Original & 0.07 & 0.14 & -- & 0.41 & 0.31 & 0.24 & 0.80 & 0.33 & 0.32 & 0.95 & 0.82 & 0.40 & 0.41 \\
 & CRU p85 & 0.58 & 0.84 & \textbf{0.24} & 0.20 & 0.30 & 0.21 & 0.58 & 0.30 & 0.23 & 0.77 & 0.38 & 0.33 & 0.30 \\
 & CRU p90 & 0.42 & 0.71 & 0.33 & 0.27 & 0.31 & 0.24 & 0.73 & 0.33 & 0.30 & 0.81 & 0.40 & \textbf{0.37} & 0.36 \\
 & CRU p95 & 0.39 & 0.62 & 0.35 & \textbf{0.30} & \textbf{0.31} & \textbf{0.26} & \textbf{0.79} & \textbf{0.33} & \textbf{0.32} & \textbf{0.85} & \textbf{0.53} & 0.37 & \textbf{0.38} \\
 & CRU p99 & \textbf{0.72} & \textbf{0.87} & 0.32 & 0.01 & 0.30 & 0.24 & 0.13 & 0.32 & 0.27 & 0.18 & 0.11 & 0.37 & 0.08 \\
\midrule
\multirow{5}{*}{DeepSeek} & Original & 0.25 & 0.33 & -- & 0.79 & 0.32 & 0.28 & 0.89 & 0.32 & 0.28 & 0.67 & 0.64 & 0.42 & 0.42 \\
 & CRU p85 & \textbf{0.79} & \textbf{0.89} & \textbf{0.28} & \textbf{0.83} & 0.31 & 0.25 & 0.87 & 0.31 & 0.25 & 0.62 & 0.21 & 0.38 & 0.35 \\
 & CRU p90 & 0.68 & 0.87 & 0.31 & 0.76 & 0.32 & 0.27 & 0.87 & \textbf{0.32} & \textbf{0.28} & 0.62 & \textbf{0.22} & 0.41 & \textbf{0.36} \\
 & CRU p95 & 0.57 & 0.88 & 0.32 & 0.80 & \textbf{0.32} & \textbf{0.27} & \textbf{0.88} & 0.32 & 0.28 & \textbf{0.66} & 0.16 & \textbf{0.43} & 0.35 \\
 & CRU p99 & 0.68 & 0.83 & 0.34 & 0.63 & 0.32 & 0.27 & 0.85 & 0.31 & 0.27 & 0.52 & 0.17 & 0.40 & 0.34 \\
\bottomrule
\end{tabular}}
\end{table}

For the \texttt{TOFU} benchmark, we adhere to the recommended hyperparameters in the paper (\cite{maini2024tofu}) for all baselines.  For \our, we perform a small hyperparameter study on $p$ in the setting of \texttt{TOFU} \texttt{forget01}, illustrated in \Cref{fig:cru_percentile_ablation}. All metrics for this same ablation are reported in \Cref{tab:tofu_99_cru_ablations}.

For \texttt{RWKU}, we strictly adhere to the best hyperparameters found by PURGE \cite{zaradoukas2026reinforcement} for all baselines to ensure a fair comparison. The paper itself \cite{cao2024rwku} reports the number of warmup steps (20) before applying unlearning. For PURGE, since we use the released checkpoints, there are no hyperparameters to tune. For \our, we perform a small grid search on the learning rate and the value of $\lambda$ on a 3-identities held-out set. We sweep $\lambda$ at the default $p = 95$ and then sweep $p$ at the $\lambda$ the first sweep picks.

\section{Full Results on benchmarks}\label{apx:full_results}

\subsection{TOFU}



\Cref{fig:forget_quality} plots the Forget Quality from \texttt{TOFU} across all models and settings. The forget quality is essentially the p-value of the KS test against the Gold Model (see \Cref{tab:benchmark_metrics}). On \texttt{forget01}, \our\ achieves a Forget Quality above the $p=0.05$ threshold for all models except Llama (which is still very close). Only NPO and SimNPO can equal or surpass \our\ in this metric. On \texttt{forget10}, specifically Phi, \our\ and GD are the only methods able to surpass the threshold. 

\begin{figure}
    \centering
\includegraphics[width=\linewidth]{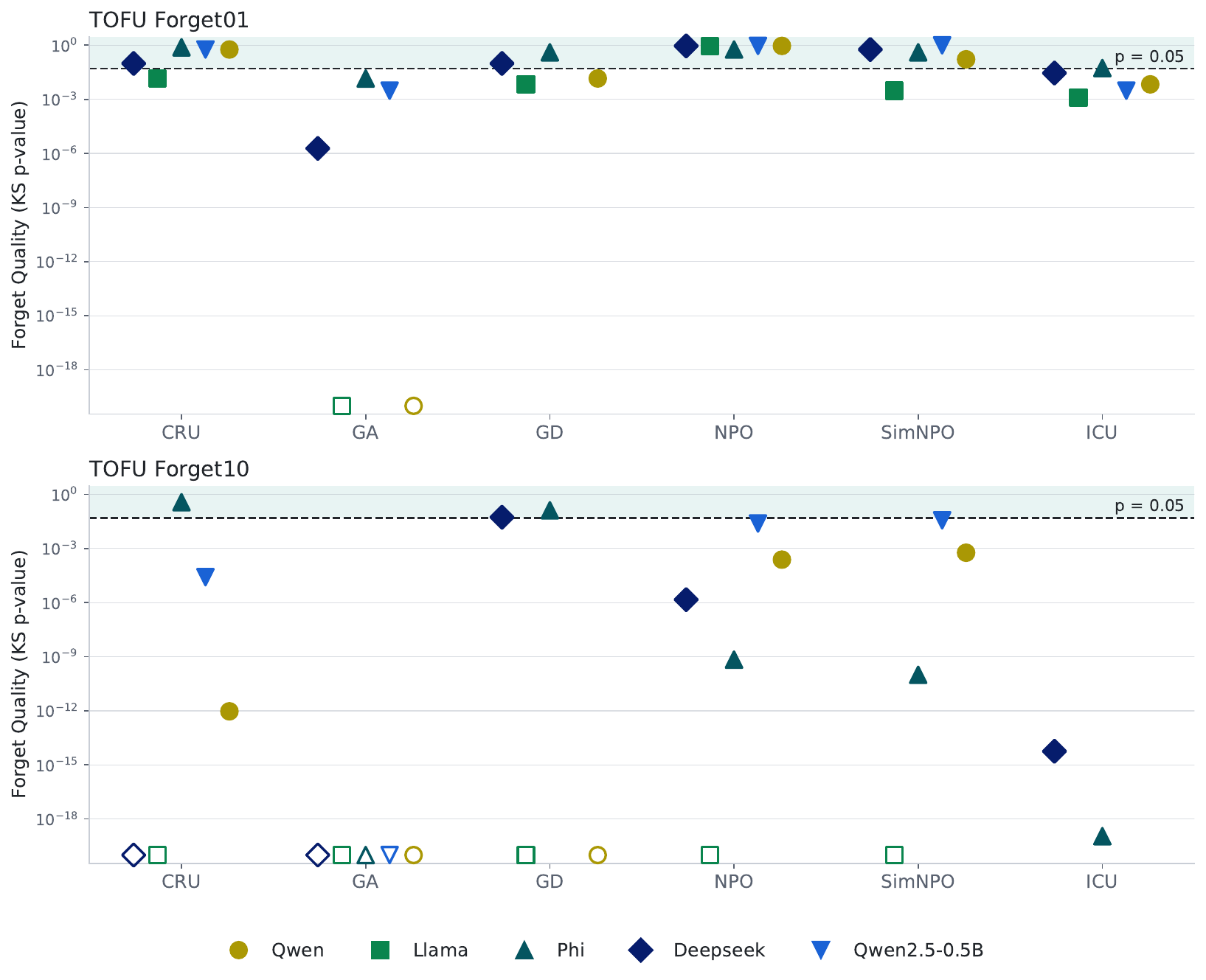}
    \caption{\texttt{TOFU}'s Forget Quality across all models and settings. Values above $\geq 0.05$ are statistically indistinguishable from the Gold Model.}
\label{fig:forget_quality}
\end{figure}

\subsection{RWKU}

\begin{figure}[t]
  \centering
  \includegraphics[width=\textwidth]{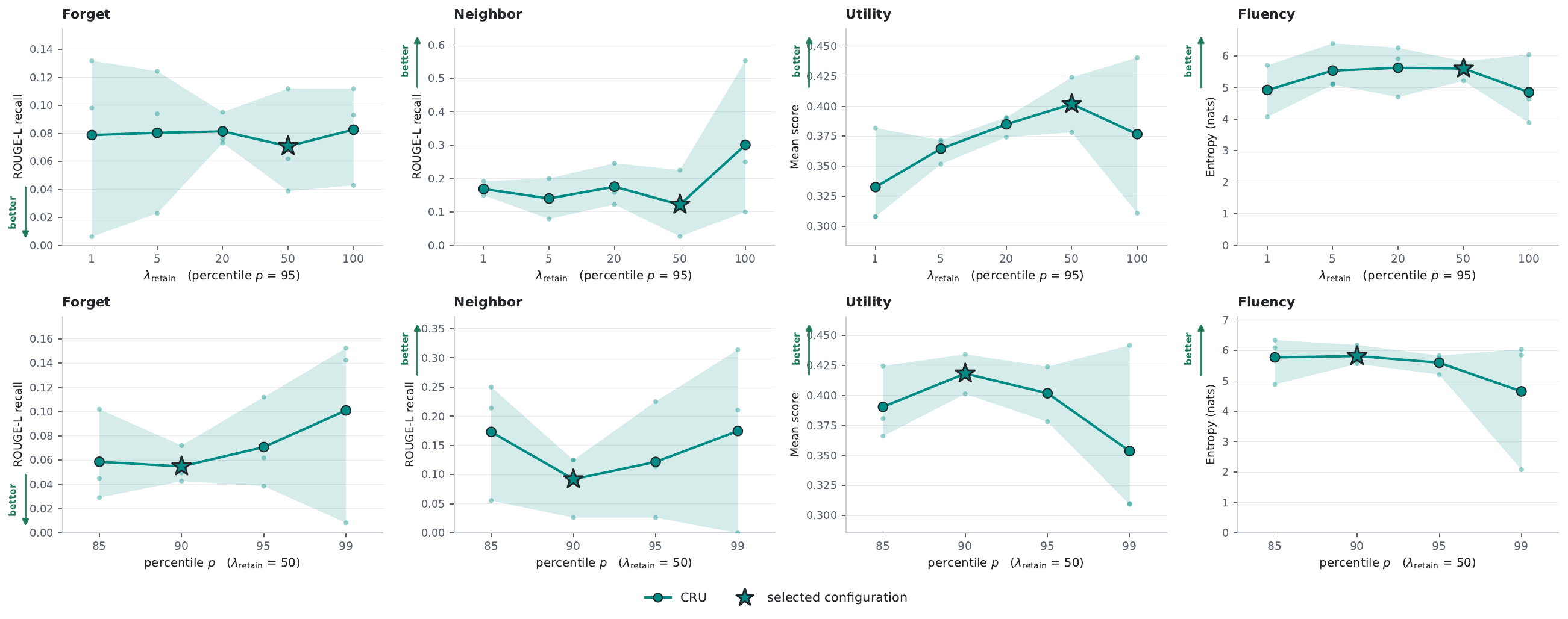}
  \caption{Hyperparameter selection for CRU on \texttt{RWKU} on Phi-3, measured on the three held-out targets (Beyonc\'e, Evel Knievel, Jim Morrison) that are disjoint from the test set. Top row: $\lambda$ at percentile $p = 95$. Bottom row: $p$ at $\lambda = 50$. The line is the mean over the three targets, the dots are the targets themselves, and the band is their range.}
  \label{fig:rwku-hparams}
\end{figure}

Figure \ref{fig:rwku-hparams} shows the hyperparameter selection for RWKU on Phi-3. Forgetting is saturated everywhere. Mean forget ROUGE-L stays between $0.055$ and $0.101$ across all nine configurations, compared to $0.484$ for the original model, so it does not separate them, and the choice has to be made based on what that forgetting costs.

Along $\lambda$, mean utility rises from $0.333$ at $\lambda = 1$ to $0.402$ at $50$ and then falls back to $0.377$ at $100$. 

Along $p$, the value $90$ is the best point on both axes simultaneously, with the lowest forget score ($0.055$) and the highest utility ($0.418$). The failure at $p = 99$ is the informative end of that sweep. Gating a fifth as many neurons does not make the intervention gentler, because the router compensates by closing harder on the units it still has, resulting in worse forgetting ($0.101$) and lower utility ($0.354$). Following this analysis, we select $\lambda = 50$ and $p = 90$ for testing the full set on RWKU for Phi-3.

\begin{figure}[t]
  \centering
  \includegraphics[width=\textwidth]{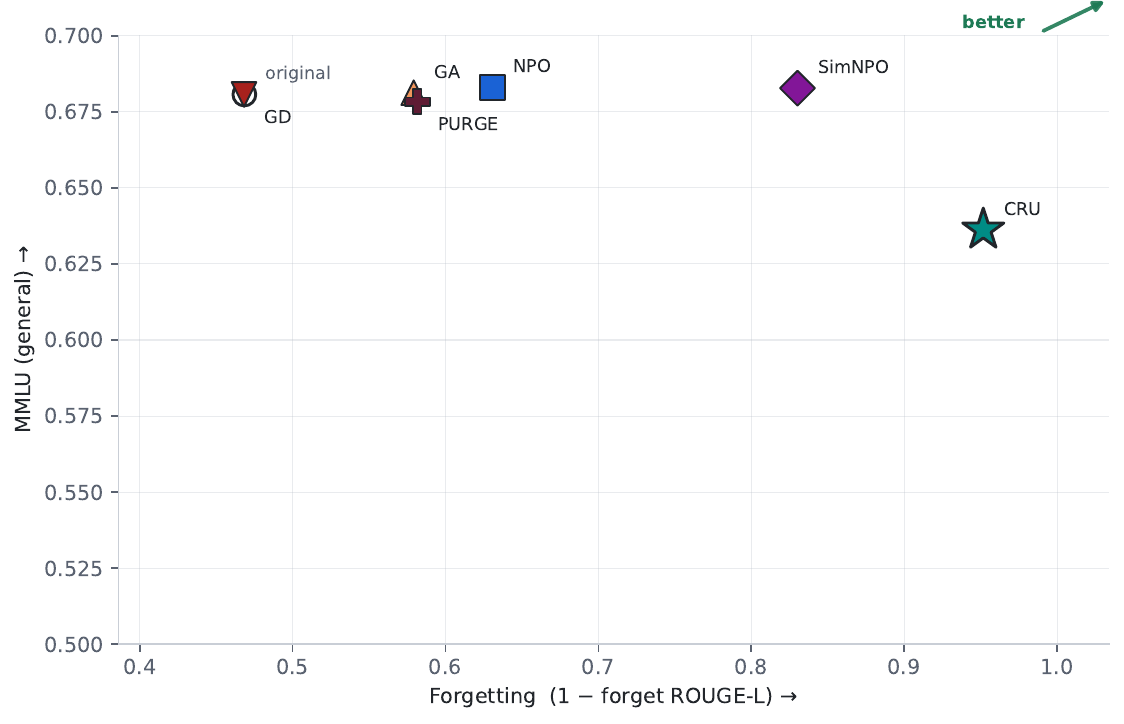}
  \caption{Forgetting against retained general knowledge on \texttt{RWKU} on Phi-3, both as means over the twenty targets. The open circle is the original model.}
  \label{fig:rwku-tradeoff}
\end{figure}

\Cref{fig:rwku-tradeoff} shows the tradeoff. \our\ is alone on the right at $0.952$ forgetting, with SimNPO next at $0.830$ and everything else between $0.468$ and $0.631$. It pays for that with MMLU, at $0.636$ against an original $0.681$, while the baselines all stay within a point of the original. No method dominates \our, in the sense that none reaches both more forgetting and higher MMLU.

\our\ takes 24~minutes per target, against 52 for NPO, 43 for GA and SimNPO, and 15 for GD (\Cref{fig:lunar-cost}). Of \ our 24 minutes, localization accounts for about half a minute: the concept neurons come from a single forward pass over the forget corpus and train nothing, and the rest is spent fitting the routing modules. PURGE is not present here because we fetch the published checkpoints. Please refer to the original PURGE paper for running times \citep{zaradoukas2026reinforcement}.

\subsubsection{Resources}

\begin{figure}[t]
  \centering
  \includegraphics[width=\textwidth]{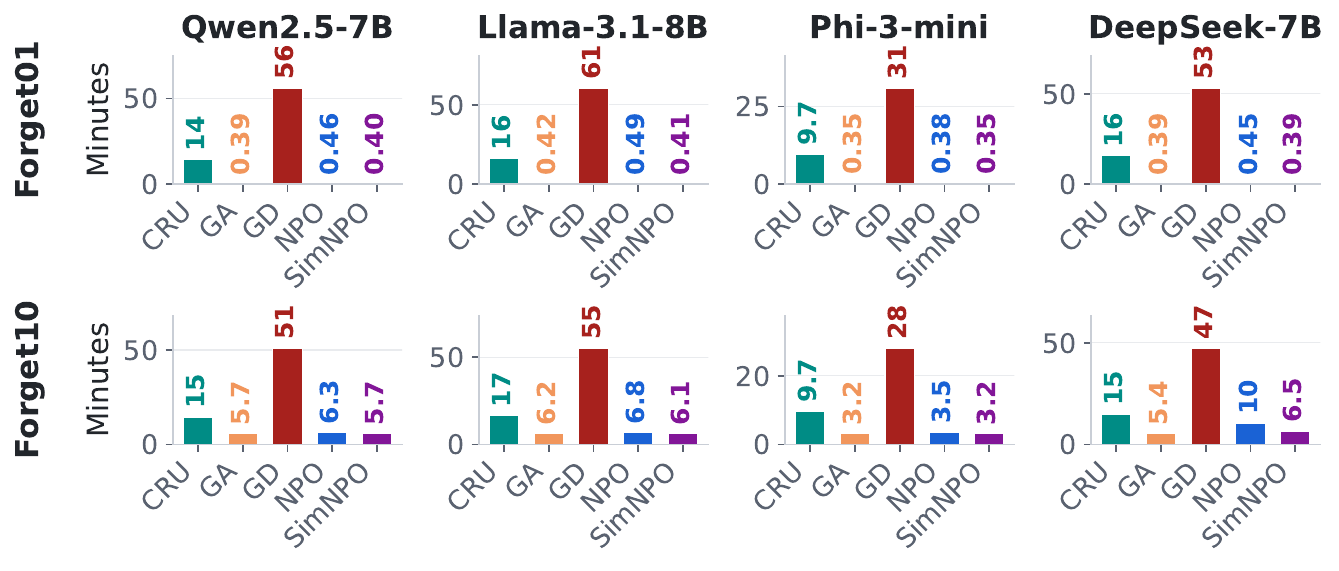}
  \caption{Runtime on both \texttt{TOFU} splits for all models.}
  \label{fig:full-runtime}
\end{figure}

\Cref{fig:full-runtime} shows the full runtime results on both \texttt{TOFU} splits across all models. \our\ is several times faster than GD, but slower than other baselines, especially on \texttt{Forget01}. However, as shown in \S\ref{sec:results}, \our\ is faster than most baselines on RWKU, and uses much less memory across all benchmarks. Moreover, using \our\ is still orders of magnitude faster than retraining an LLM from scratch.
\subsection{Full Ablations}\label{apx:ablations}

\begin{figure}[t]
  \centering
  \includegraphics[width=\textwidth]{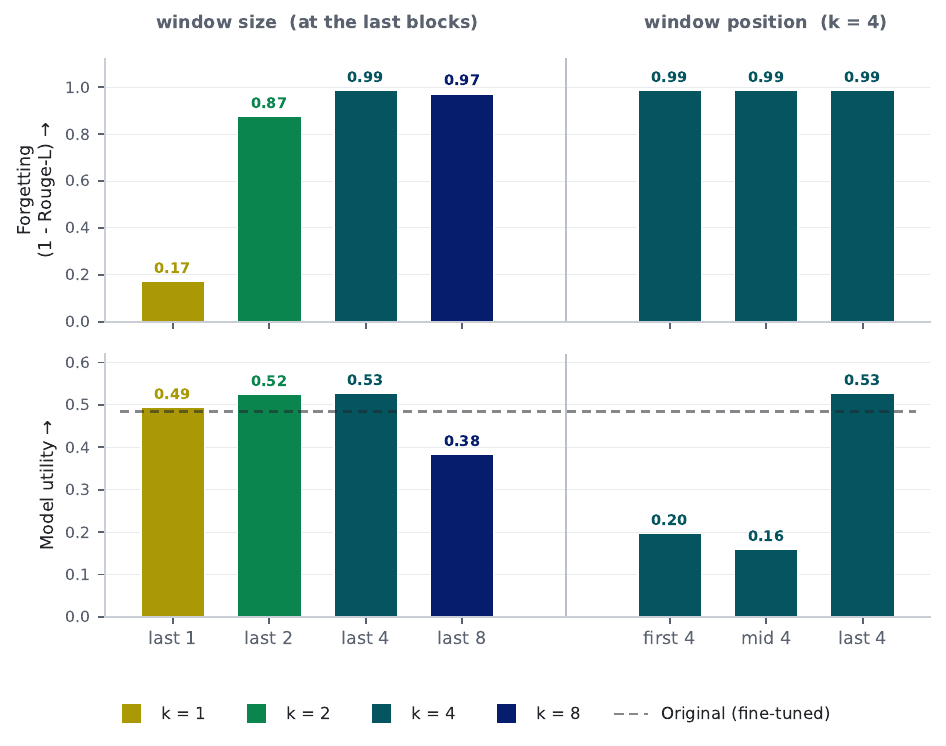}
  \caption{Ablation on the position and number of blocks where \our\ selects neurons.}
  \label{fig:block-ablation}
\end{figure}

\subsubsection{Overlap of high-variance neurons}\label{apx:overlap}

Phase 1 of \our\ selects neurons by the empirical variance $\sigma_k^2(j)$ of their activations over the forget set. Variance, however, is in part a property of a neuron. A natural objection is therefore that \our\ finds the generically high-variance units of layer $k$ regardless of what we asked the model to forget. This would mean that the gate in Phase 2 succeeds by suppressing them rather than by suppressing the concept. 

To disprove this, we run the following experiment: we run Phase 1 twice on the same fine-tuned weights, once on $40$ forget-set examples and once on $40$ retain-set examples, and measure the overlap in neurons that appear in both lists.  Write $C_k^{f}$ for the concept neuron set that \our\ computes, over $n_f = 40$ forget-set examples, and $C_k^{r}$ for the set that the identical procedure returns when those examples are replaced by $n_r = 40$ retain-set examples, with the same frozen fine-tuned weights, the same layer $k$, and the same $\alpha$. We aim to compute the overlap coefficient $|C_k^{f} \cap C_k^{r}| / |C_k^{f}|$.

\Cref{fig:concept-overlap} shows the results. Averaged over the $K$ layers, the two lists share $27\%$ of their neurons on Llama, $38\%$ on DeepSeek, $46\%$ on Phi, and $59\%$ on Qwen: well short of being the same list. The same ordering appears without any threshold at all, in the rank Spearman correlation between the two sets of variances ($\rho = 0.18$, $0.37$, $0.44$, $0.46$). Thus, we can claim that the vast majority of the neurons that \our\ gates are specific to the forget set; a sizable minority is not, and depends on the inherent variance of each neuron.

\subsubsection{Ablation on Block size and position}

\Cref{fig:block-ablation} shows the ablation on the number and position of blocks where \our\ selects the most variant neurons. The ablation was run on \texttt{TOFU} \texttt{forget10} on Qwen-2.5-7B. On the left column, Forgetting and Model Utility when selecting the last $\{1,2,4,8\}$ blocks to select neurons. Our choice of $k=4$ is the best both in Forgetting and Model Utility. Clearly, selecting more blocks (e.g., 8) means that neurons in earlier layers (and thus encoding more general knowledge) are gated and suppressed, thereby hindering the model's utility.

On the right, Forgetting and Model Utility when the window position changes. While Forgetting remains roughly the same (possibly due to model collapse and catastrophic forgetting), selecting the first 4 blocks or the middle 4 blocks severely hurts model utility, for the same reason described before. As such, we found that selecting the last $k=4$ layers is optimal.

\begin{figure}[t]
  \centering
  \includegraphics[width=\textwidth]{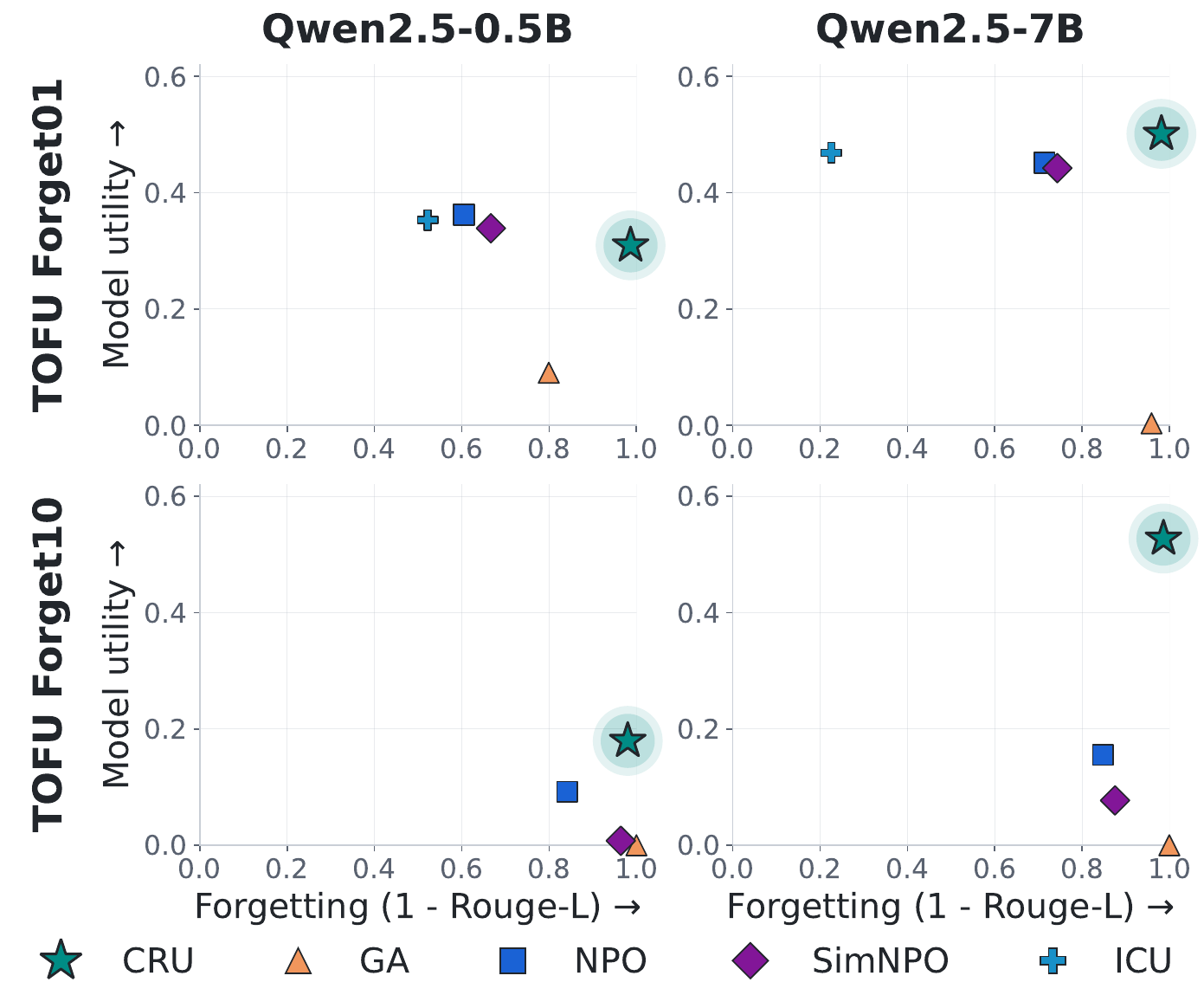}
  \caption{Ablation on the size of the model, using Qwen-2.5 with 7 billion parameters and half a billion parameters.}
  \label{fig:size-ablation}
\end{figure}

\subsubsection{Ablation on Model Size}

We also ablate the model size to see if \our\ is succeeding only on big models, or if results can be generalized. To test this, we employ the same Qwen 2.5 model in both its 7-billion-parameter (used in our experiments) and 0.5-billion-parameter versions. We perform the test on both \texttt{TOFU} \texttt{forget01} and \texttt{forget10}. Results are visualized in \Cref{fig:size-ablation}, following the same schema as \Cref{fig:pareto_grid}.

\our\ has the best Forgetting-Utility tradeoff across all settings, for both models, and both \texttt{TOFU} splits, surpassing all baselines. The success of \our\ is therefore not merely a byproduct of model size, but can be generalized to smaller models too.



\end{document}